\documentclass{article}

\usepackage[preprint]{neurips_2026}

\usepackage[utf8]{inputenc} 
\usepackage[T1]{fontenc}    
\usepackage{hyperref}       
\usepackage{url}            
\usepackage{booktabs}       
\usepackage{amsfonts}       
\usepackage{nicefrac}       
\usepackage{microtype}      
\usepackage{xcolor}         
\usepackage{graphicx}
\usepackage{amsmath}
\usepackage{amssymb}
\usepackage{amsthm}
\usepackage{enumitem}
\usepackage{wrapfig}
\usepackage{multirow}
\usepackage{booktabs}
\usepackage{caption}
\usepackage{algorithm,algpseudocode}
\usepackage{booktabs}
  
\newtheorem{theorem}{Theorem}[section]

\newtheorem{lemma}[theorem]{Lemma}
\newtheorem{assumption}[theorem]{Assumption}
\newtheorem{remark}[theorem]{Remark}

\title{Variance-Optimal State Resampling for Reinforcement Learning with Verifiable Rewards}

\author{%
  Yibo Wang \qquad Yuanxin Liu \qquad Zhaoran Wang \\
  Northwestern University \\
  \texttt{\{yibowang2030, yuanxinliu2029\}@u.northwestern.edu}\\
  \texttt{zhaoranwang@gmail.com}
}

\begin{document}

\maketitle
\begin{abstract}
Reinforcement learning with verifiable rewards spends much of its rollout budget on states that yield no learning signal. We recast this as a question about the sampling distribution itself: under a fixed rollout budget, which distribution over states minimizes the variance of the policy-gradient estimator? The answer admits a closed-form Neyman allocation, $\mu_{\mathrm{var}}^\pi(s) \propto d^\pi(s)\,\sigma_\pi(s)$, which connects naturally to GRPO's group-relative advantage and induces an emergent reverse curriculum as the policy improves, with no externally imposed schedule. Building on this principle, we propose VSR (Variance-driven State Resampling), a plug-in modification to GRPO using a prefix replay buffer and a zero-cost variance proxy. On mathematical reasoning benchmarks and long-horizon agentic tasks (Sokoban, BabyAI), VSR outperforms GRPO and GiGPO on most benchmarks with fewer generated tokens per update.
\end{abstract}

\section{Introduction}
\label{sec:intro}
Training large language models with reinforcement learning from verifiable rewards has become a standard recipe for improving reasoning and agentic capabilities. In practice, however, a non-trivial fraction of training steps contribute little signal: in GRPO-style training, rollout groups that return uniform rewards—all-fail or all-success—yield a collapsed group-relative advantage and a near-zero update.

\begin{wrapfigure}{r}{0.5\textwidth}
  \vspace{-0.8em}
  \centering
  \includegraphics[width=1\linewidth]{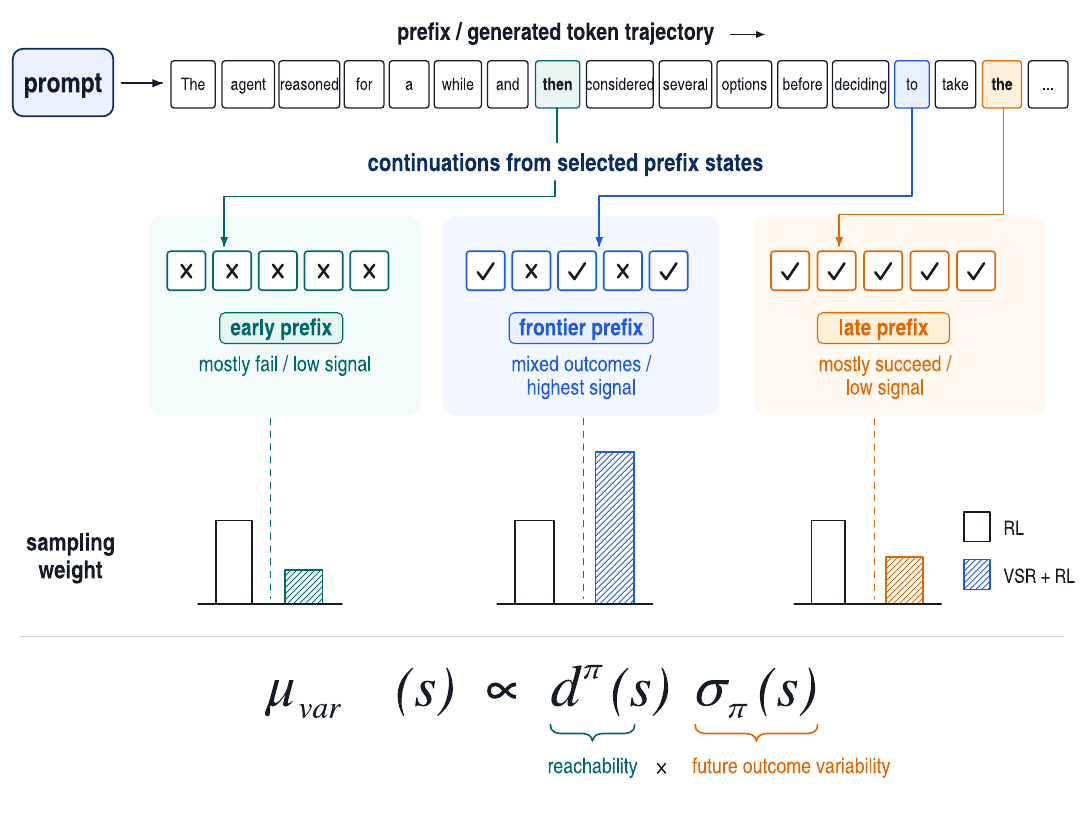}
    \caption{Continuations from early prefixes mostly fail and from late prefixes mostly succeed; both yield low signal. VSR samples states in proportion to $d^\pi(s)\,\sigma_\pi(s)$ to capture the highest signal.}
    \label{fig:first_big_graph}
  \vspace{-1.0em}
\end{wrapfigure}

It can be helpful to look at this pathology at the level of state distribution, rather than prompt difficulty. Even on hard problems, not all intermediate states are equally uninformative. Consider a competition math problem that the current policy rarely solves from the prompt alone. Once the policy has produced a partial chain of thought that performs a key algebraic reduction, continuations from that prefix tend to succeed with noticeably higher probability, and different continuations produce genuinely mixed outcomes. A similar pattern shows up in long-horizon agentic tasks: from the initial query, a successful trajectory is uncommon, but once the agent has identified the relevant information source through a few tool calls, the remaining decisions can lead to meaningfully different outcomes. In both cases, the trajectory passes through intermediate states with rich gradient signals, although rollouts generated from the prompt alone rarely reach it. The challenge is that such states tend to be visited with low probability under the current policy's state distribution $d^\pi$, so finite-sample training encounters them only occasionally.

Therefore, a complementary question which seems to have received comparatively less direct attention is: \textit{Under a fixed rollout budget, which state distribution minimizes the variance of the policy-gradient estimator?}


In this paper we take up that question. Under a gradient uniformity assumption, we find that the answer takes a closed form with Neyman-allocation structure,
\begin{equation}
\mu_{\mathrm{var}}^\pi(s) \propto d^\pi(s)\sigma_\pi(s).
\end{equation}
This proposal allocates training mass by the product of on-policy relevance and residual signal magnitude, and it evolves with the policy, giving rise to an emergent curriculum without a hand-designed difficulty schedule.

 Implementing $\mu_{\mathrm{var}}^\pi$ with exact proposal correction would involve an inverse-$\sigma_\pi$ factor. GRPO’s group-relative normalization provides a practical analogue of this factor by normalizing rewards with an empirical within-group standard deviation. We instantiate this idea as \textbf{VSR (Variance-driven State Resampling)}, a practical plug-in modification to GRPO-style algorithms that approximates the variance-optimal proposal. As a tractable algorithmic realization of the allocation principle, VSR uses a recent prefix replay buffer as a density-free surrogate for the current state-visitation distribution and a cheap proxy for $\sigma_\pi$ to prioritize replayed prefixes. Experiments on math reasoning and long-horizon agentic tasks demonstrate the effectiveness of this approximation.

\paragraph{Contributions.}
\begin{itemize}
    \item We formulate the uniform-reward pathology in GRPO-style training as a variance-reduction problem over the state sampling distribution, and derive a closed-form $\mu_{\mathrm{var}}^\pi \propto d^\pi(s)\,\sigma_\pi(s)$, which matches the form of GRPO's group-relative advantage and induces an emergent reverse curriculum as the policy improves.
    
    \item Building on this principle, we propose VSR (Variance-driven State Resampling), a plug-in practical modification to GRPO that approximates $\mu_{\mathrm{var}}^\pi \propto d^\pi(s)\,\sigma_\pi(s)$ through a state replay buffer and a zero-cost variance proxy.
    
    \item On hard mathematical reasoning benchmarks and long-horizon agentic tasks, VSR outperforms GRPO and GiGPO baselines with fewer generated tokens per update, and ablations isolate the contributions of variance-driven scoring and recency-based buffer updates.
\end{itemize}

\section{The variance-driven state proposal}
\subsection{Setup}
\label{sec:setup}

A natural starting question for any policy-gradient method is: what distribution is training actually averaging over? For general RL the answer is the state visitation measure $d^\pi$, and the policy gradient takes the form
\begin{equation}
    \nabla J(\pi) \;=\; \mathbb{E}_{s \sim d^\pi,\, a \sim \pi(\cdot \mid s)}\!\left[\nabla \log \pi(a \mid s)\, \tilde{r}(s, a)\right], \quad d^\pi(s) \;\propto\; \sum_{t=0}^{T-1} \Pr\bigl(s_t = s \,\big|\, s_0 \sim \mathcal{D},\, \pi\bigr),
\end{equation}
where $\tilde{r}(s, a) = r(s, a) - b(s)$ for a state-dependent baseline $b$. The expectation is taken with respect to $d^\pi$, which is therefore the distribution under which gradient signal is sampled.

LLM RL with verifiable rewards is most often presented as a contextual bandit under which the only input-side distribution is the prompt distribution $\mathcal{D}$. This framing is coarse relative to how generation actually proceeds: a single rollout on a reasoning problem spans thousands of tokens and passes through qualitatively distinct regimes such as restatement, algebraic reduction, and reflection, while on long-horizon agentic tasks a trajectory decomposes into a sequence of consequential intermediate decisions. We therefore adopt the standard state-action view throughout, in which
\begin{equation}
    s_0 = x \sim \mathcal{D}, \qquad s_t = (x, y_{<t}), \qquad a_t = y_t,  \qquad \tau = (s_0, a_0, s_1, a_1, \ldots, s_T).
\end{equation}
This is the same structure realized by the triangular attention mask inside the transformer itself. Under this instantiation, $d^\pi$ is a marginal over all intermediate contexts the policy passes through, not over prompts alone.


\subsection{The sampling proposal as a variance-reduction lever}
\label{sec:proposal}

Reducing the variance of policy-gradient estimators is a long-standing concern, and most of the work has gone into the gradient baseline $b(s)$ \citep{greensmith2004variance}. The proposal $\mu$ from which states are sampled has not, by comparison, been treated as something to be optimized: it is set to $d^\pi$ and any deviation is handled as an off-policy correction problem rather than as a variance-reduction choice. This section asks whether $\mu$ can be optimized alongside the baseline to minimize the variance.

The standard policy-gradient estimator with a state-dependent baseline takes
the form
\begin{equation}
\label{eq:pg-baseline}
\widehat{\nabla J}_b(\pi) \;=\; \big(r(s,a)-b(s)\big)\,\nabla\log\pi(a\mid s),
  \qquad s\sim d^\pi,\ a\sim\pi(\cdot\mid s),
\end{equation}
which is unbiased for $\nabla J(\pi)$ for every state-dependent $b$.
For trace-covariance minimization, however, the optimal baseline generally depends on the action-gradient norm. Writing $g(s,a):=\nabla\log\pi(a\mid s)$ and $H(s,a):=\|g(s,a)\|^2$, the statewise optimum is
\begin{equation}
\label{eq:norm-weighted-baseline-main}
b_w^*(s)=\frac{\mathbb{E}[r(s,a)H(s,a)\mid s]}{\mathbb{E}[H(s,a)\mid s]}.
\end{equation}
It reduces to the conditional mean $\bar r(s):=\mathbb{E}[r(s,a)\mid s]$ under the within-state conditional-moment condition stated below.

Crucially, unbiasedness of~\eqref{eq:pg-baseline} does not require sampling
from $d^\pi$. For any proposal $\mu$ with $\mathrm{supp}(\mu)\supseteq\mathrm{supp}(d^\pi)$,
the importance sampling estimator
\begin{equation}
\label{eq:pg-proposal}
\widehat{\nabla J}_{b,\mu}(\pi) \;=\; \frac{d^\pi(s)}{\mu(s)}\big(r(s,a)-b(s)\big)\,\nabla\log\pi(a\mid s),
  \qquad s\sim\mu,\ a\sim\pi(\cdot\mid s),
\end{equation}
is unbiased for every admissible pair $(b,\mu)$. The on-policy choice
$\mu=d^\pi$ is the operational default in policy-gradient methods, but it is
not, in general, the variance-minimizing choice. The reason is that states
contribute different amounts of gradient signal per sample: a state whose continuations almost always succeed or fail yields little reward variation after baseline subtraction. A state whose continuations
produce mixed outcomes yields per-sample signal proportional to that
variation. Sampling states by visitation frequency alone does not reflect
this difference.

We therefore treat $\mu$, alongside $b$, as a variable to be optimized.
Minimizing the trace covariance of~\eqref{eq:pg-proposal} jointly over
$(b,\mu)$ admits a closed-form solution.

\noindent

\begin{theorem}[Joint-optimal baseline and proposal; informal]
\label{thm:main-informal}
Let $H(s,a)=\|\nabla\log\pi(a\mid s)\|^2$ and let $b_w^*(s)$ be given by~\eqref{eq:norm-weighted-baseline-main}. The general joint minimizer of the trace covariance of~\eqref{eq:pg-proposal} is
\begin{equation}
\label{eq:joint-opt-main-general}
\mu^\star(s) \;\propto\; d^\pi(s)\sqrt{M_w^*(s)},
\qquad
M_w^*(s):=\mathbb{E}\!\left[(r(s,a)-b_w^*(s))^2H(s,a)\mid s\right].
\end{equation}
If $\mathbb{E}[H\mid s,r]=G(s)$ is independent of the realized reward within each state, then $b_w^*(s)=\bar r(s)$ and
\begin{equation}
\label{eq:joint-opt-main-state}
\mu^\star(s)\propto d^\pi(s)\sigma_\pi(s)\sqrt{G(s)}.
\end{equation}
Under the additional across-state homogeneity condition $G(s)\equiv G_0$, this reduces to
\begin{equation}
\label{eq:joint-opt-main}
\mu^\star(s) \;\propto\; d^\pi(s)\,\sigma_\pi(s) \;=:\; \mu^\pi_{\mathrm{var}}(s),
\end{equation}
where $\sigma_\pi(s):=\sqrt{\mathrm{Var}_{a\sim\pi(\cdot\mid s)}(r(s,a))}$. On its effective support, the corresponding proposal-corrected estimator takes the standardized-residual form
\begin{equation}
\label{eq:joint-opt-estimator-main}
\widehat{\nabla J}^{\,\star}(\pi) \;\propto\; \frac{r(s,a)-\bar r(s)}{\sigma_\pi(s)}\,\nabla\log\pi(a\mid s),
\qquad s\sim\mu^\pi_{\mathrm{var}},\ a\sim\pi(\cdot\mid s).
\end{equation}
\end{theorem}

\paragraph{Specialization to the GRPO form.}

Under the two conditions above, the exact proposal-corrected estimator~\eqref{eq:joint-opt-estimator-main} has the same standardized reward-residual form as GRPO's group-relative advantage~\citep{shao2024deepseekmath}. This connection concerns the form of the normalization: standard GRPO does not sample states from $\mu^\pi_{\mathrm{var}}$ or apply the exact state-density ratio $d^\pi/\mu$. A formal
statement, with an explicit expression for the variance gap, is given in
Theorem~\ref{thm:clean}.

\paragraph{Beyond the regularity assumption.}

The two conditions that yield $\mu^\pi_{\mathrm{var}}\propto d^\pi\sigma_\pi$ are idealizations. In practice, action-gradient norms can depend on the realized reward within a state and their conditional scale can vary across states, so $\mu^\pi_{\mathrm{var}}$ need not coincide with the exact joint optimum. We make this precise in Appendix~\ref{app:weak}
(Assumption~\ref{ass:G-prime} and Theorem~\ref{thm:weak}). That result is a proposal comparison under the stated regularity condition, not an unrestricted joint-optimality result.

\paragraph{Analysis of an idealized theoretical algorithm.}
 Separately from practical VSR, Appendix~\ref{app:svmd} analyzes an idealized smoothed state-reweighting algorithm under oracle-score, coverage, critic-realizability, and exact-weighting assumptions. Consider sampling critic data from a
smoothed version of $d^\pi\sigma_\pi$ at every iteration and updating the
policy by KL-regularized mirror ascent. For any comparator policy $\mu$ covered by every
iterate, the resulting theoretical algorithm satisfies
\[
J(\mu) - J(\bar\pi_T) \;\leq\; \widetilde{O}\!\left(
\sqrt{\frac{\mathrm{KL}(\mu\,\|\,\pi_1)}{T}}
\;+\; \sqrt{\frac{H^2}{\varepsilon N}}\right),
\]
combining a $1/\sqrt{T}$ optimization rate with a critic-estimation term
inflated only by the smoothing factor $1/\varepsilon$. The full statement,
assumptions, and proof are deferred to Appendix~\ref{app:svmd}.

\subsection{Structure of the variance-optimal proposal}
\label{sec:proposal:shape}

\textbf{$\mu_{\mathrm{var}}^\pi$ reweights the effective support of $d^\pi$.} The unsmoothed proposal assigns zero mass to states with $\sigma_\pi(s)=0$, i.e., states with no centered reward-variance signal, and therefore need not preserve the full support of $d^\pi$. When full support is required, one may use the smoothed proposal $\mu_\varepsilon=(1-\varepsilon)\mu_{\mathrm{var}}+\varepsilon d^\pi$, which retains an on-policy component while preserving the same variance-driven preference on the remaining mass.

\textbf{$\mu_{\mathrm{var}}^\pi$ trades off on-policy relevance and residual signal.} The two factors capture distinct and complementary quantities: $d^\pi(s)$ measures how relevant a state is to the current policy's behavior, penalizing distributional shift away from on-policy sampling, while $\sigma_\pi(s)$ measures how much per-sample gradient signal a state carries, penalizing states whose outcomes are already deterministic. Their product is the Neyman allocation that minimizes estimator variance---what is upweighted is the intersection of states that are both reachable under $\pi$ and informative about $\pi$'s residual errors.

\begin{figure}[h]
    \centering
    \includegraphics[width=0.8\linewidth]{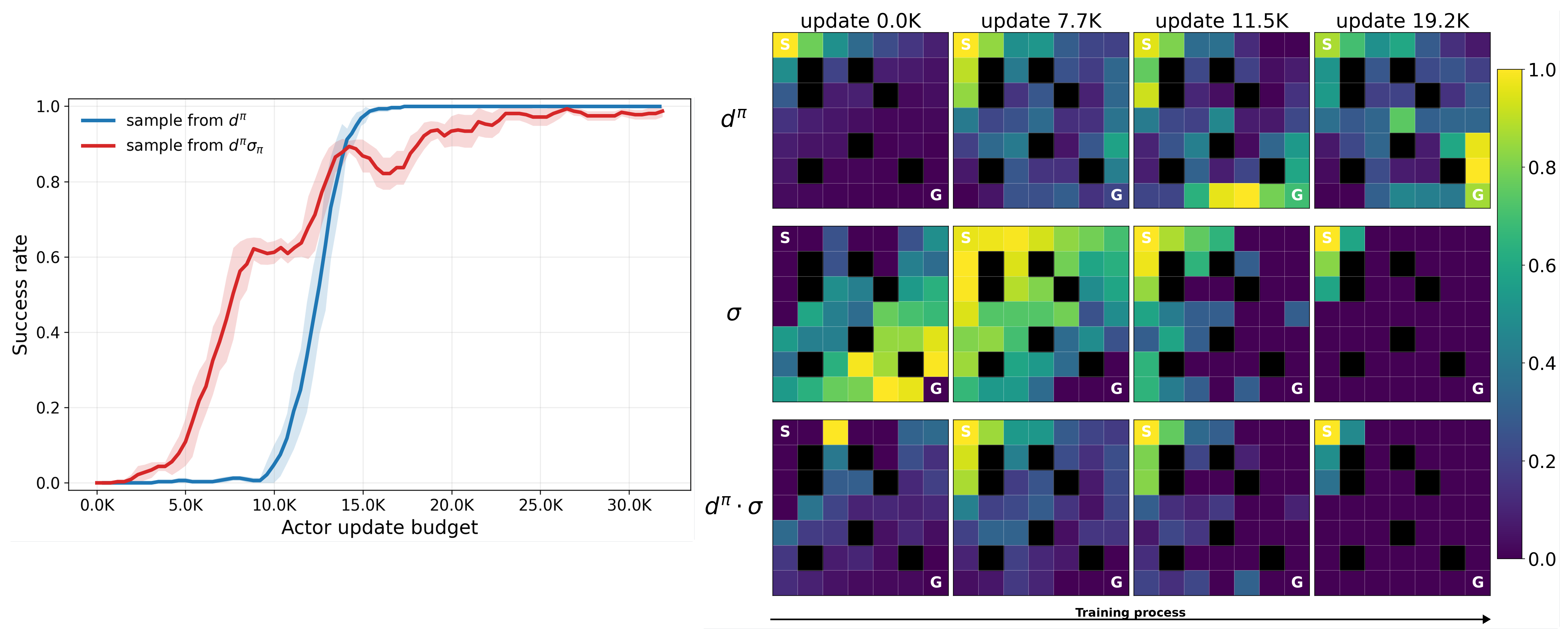}
    \caption{\textbf{An emergent curriculum from the variance-optimal proposal.}  $\mu_{\mathrm{var}}^\pi$ is concentrated near $G$ early in training, and gradually walks backward through the maze toward $S$ as cells closer to the goal are mastered one after another. This is the reverse curriculum: the training signal visits the state space in reverse without any externally imposed schedule. It is a built-in property of variance-optimal sampling: the variance-minimizing distribution by construction concentrates wherever the policy is reachable but still uncertain, and that frontier moves backward as the policy improves. }
    \label{fig:toy_model}
\end{figure}

\textbf{$\mu_{\mathrm{var}}^\pi$ shifts with the policy, producing an emergent curriculum.} Both factors depend on $\pi$: states that were once unreachable acquire visitation mass as $\pi$ improves, and states that once produced mixed-outcome rollouts become near-deterministic and lose residual signal. The frontier of high $\mu_{\mathrm{var}}^\pi$ therefore re-localizes itself as training proceeds, without any externally imposed schedule.

The dynamics above can be hard to visualize in settings where $d^\pi$ and $\sigma_\pi$ are intractable, so we examine $\mu_{\mathrm{var}}^\pi$ in a setting where both can be computed exactly. A small MLP policy navigates a $7 \times 7$ gridworld with obstacles from start $S$ to goal $G$ under terminal reward $R \in \{0, 1\}$; the state space is small enough that $d^{\pi_t}(s)$ and $\sigma_{\pi_t}(s)$ can be enumerated, and $\mu^{\pi_t}_\mathrm{var} \propto d^{\pi_t} \sigma_{\pi_t}$ sampled from without approximation. Figure~\ref{fig:toy_model} shows the three quantities and a learning-curve comparison across training. Early on, $d^{\pi_t}$ concentrates near $S$ while $\sigma_{\pi_t}$ peaks in a band of states near $G$ that the policy can occasionally reach; their product picks out the intersection. As training progresses, the high-$\sigma$ band recedes from $G$ toward $S$ as goal-adjacent states become near-deterministic successes, and $\mu_t$ tracks it---until both factors collapse onto a small residual hard region near the start. We note that this curriculum is not designed; it is what minimizing estimator variance looks like when made visible.

\section{VSR: A Practical Algorithm}
\label{sec:vsr}

The previous section characterized the variance-optimal proposal 
$\mu_{\mathrm{var}}^\pi \propto d^\pi \sigma_\pi$ as a structural object. 
However, both $d^\pi$ and $\sigma_\pi$ are intractable to compute exactly 
under realistic LLM rollout budgets, requiring cheap surrogates for 
implementation. We present \textbf{Variance-driven State Resampling} (VSR) in Algorithm~\ref{alg:vsr}, a plug-in modification of GRPO that approximates 
$\mu_{\mathrm{var}}^\pi$ via a prefix replay buffer and a zero-cost 
variance proxy.
The resulting algorithm should be understood as an 
approximation to the allocation principle of $\mu_{\mathrm{var}}^\pi$ 
rather than as an unbiased estimator of the exact proposal-corrected gradient it induces. In particular, practical VSR does not estimate the density ratio $d^\pi/\mu$: it changes the reset-state distribution from which rollout groups are launched and then applies the usual GRPO-style loss. Because the replay buffer and priorities evolve with the policy and replay history, practical VSR is a biased adaptive reset-state procedure motivated by the variance-allocation geometry, not optimization of one fixed reweighted objective.

\begin{algorithm}[h]
\caption{VSR (Variance-driven State Resampling)}
\label{alg:vsr}
\begin{algorithmic}[1]
\Require Initial policy $\pi_0$, prompt distribution $\mathcal{D}$, group size $G$, recency window $W$, replay ratio $\rho \in [0,1]$
\Statex \textbf{Optional:} expert successful trajectories $E$ for cold-start bootstrapping
\State Initialize replay buffer $\mathcal{B}_0 \gets \varnothing$ (or seed with prefix states from $E$ if cold-starting)
\For{$t = 0,1,\dots,T-1$}
    \State Draw fresh prompts $\{x_j\} \sim \mathcal{D}$ (fraction $1-\rho$) and replay prefixes $\{s_k\} \sim \mathcal{B}_t$ with probability $\propto \hat{\sigma}_t(s)$ (fraction $\rho$)
    \State Generate $G$ rollouts under $\pi_t$ from each $x_j$ and each $s_k$ 
    \State Compute GRPO-style loss function $\mathcal{L}_{\mathrm{PG}}(\pi_t)$ over the rollouts
    \State Update $\pi_{t+1} \gets \pi_t - \eta\, \nabla \mathcal{L}_{\mathrm{PG}}(\pi_t)$ 
    \State For each successful fresh trajectory $\tau$ generated from prompt $x_j$, insert every prefix state $s_i = (x_j, y_{<i})$ into $\mathcal{B}_{t+1}$ with score
    \[
    \hat{\sigma}_{t+1}(s_i) = \sqrt{\hat{p}_i(1-\hat{p}_i)},
    \quad
    \hat{p}_i = (1-\alpha_i)\,p_0(x_j) + \alpha_i\, R(\tau),
    \quad
    \alpha_i = i/|\tau|
    \]
    where $p_0(x_j)$ is the empirical success rate of $x_j$'s fresh rollout group
    \State Evict from $\mathcal{B}_{t+1}$ all entries older than $W$ steps
\EndFor
\end{algorithmic}
\end{algorithm}

\subsection{From $\mu^\pi_{\mathrm{var}}$ to a tractable surrogate}

\begin{wrapfigure}{r}{0.5\textwidth}
  \centering
  \vspace{-10pt}
  \includegraphics[width=0.5\textwidth]{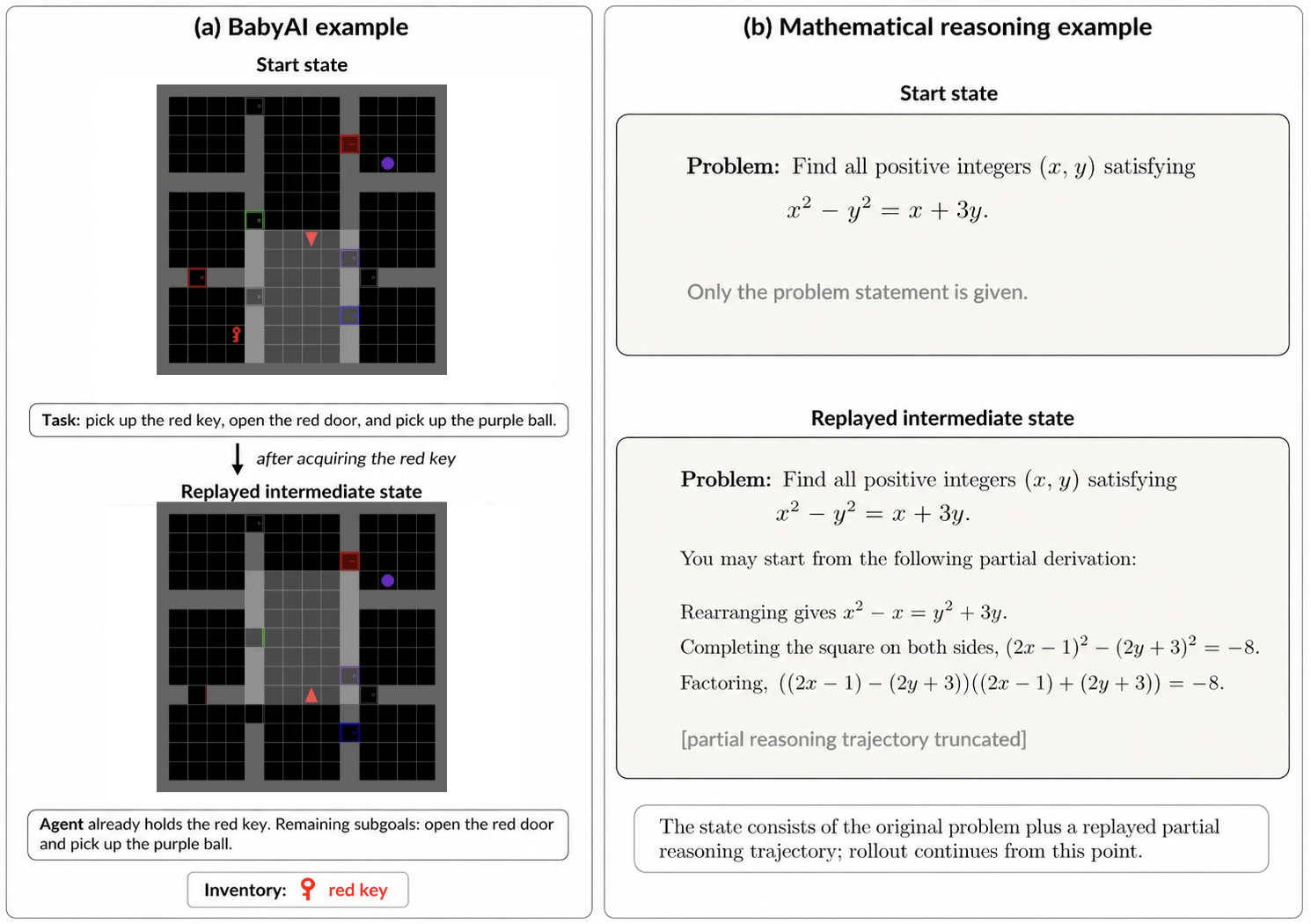}
  \caption{Illustration of intermediate states used by VSR. Left: a BabyAI agent after picking up the red key. Right: a math problem with a partial reasoning trace appended.}
  \label{fig:middlestate}
  \vspace{-10pt}
\end{wrapfigure}

The two factors in $\mu^\pi_{\mathrm{var}}$ are intractable in different ways. \textbf{(i) $\sigma_\pi$ need to be approximated at zero rollout cost.} Estimating $\sigma_\pi(s)$ from conditional rollouts would reintroduce precisely the cost that state-level resampling is designed to amortize, so any usable proxy must be built from quantities already available in each GRPO step---the group success rate $p_0$ and terminal rewards $R(\tau)$. \textbf{(ii) $d^\pi$ is not directly tractable.} We sidestep explicit density modeling via a recency-windowed prefix replay buffer $\mathcal{B}_t$ that collects states from recent rollouts, so that uniform draws from $\mathcal{B}_t$ approximate samples from $d^{\pi_t}$; recency-based eviction keeps the buffer aligned with the current policy.

Combining the two, we sample prefixes from $\mathcal{B}_t$ with probability proportional to $\hat\sigma_t(s)$ and launch fresh rollout groups from each (Algorithm~\ref{alg:vsr}). The design follows directly from $\mu^\pi_{\mathrm{var}}$: stay close to on-policy visitation through $\mathcal{B}_t$, and concentrate the sampling budget on states that still carry signal through $\hat\sigma_t$. Both surrogates are deliberately cheap rather than accurate.
\subsection{Online prefix proposal construction}

\paragraph{Recent prefix buffer.}
VSR maintains a recency-windowed replay buffer \(\mathcal{B}_t\) of token-level prefix states extracted from recent rollout groups. For a trajectory \(\tau=(s_0,a_0,\ldots,s_H)\), a replayable state has the form \(s_i=(x,y_{<i})\), i.e., the original prompt together with the generated prefix up to position \(i\). In agentic settings, the same state representation is applied to the serialized interaction history, including prior actions, tool outputs, and environment observations (Figure~\ref{fig:middlestate}). Sampling from \(\mathcal{B}_t\) gives a density-free empirical approximation to the current state-visitation measure \(d^{\pi_t}\), while recency eviction keeps the proposal aligned with the evolving policy.

In our implementation, \(\mathcal{B}_t\) stores prefixes from successful trajectories only. This is a deliberate departure from the exact empirical visitation distribution.  This filtering is not theoretically required by the variance-optimal proposal. In the hard-problem regime studied here, longer failed prefixes are empirically more likely to yield another all-fail continuation group (Appendix~\ref{app:rebuttal-controls}); we therefore use successful-prefix filtering as a practical budget-allocation choice that favors states known to lie on at least one completable trajectory.

\paragraph{Zero-cost variance score.}
Each buffered prefix is assigned a score intended to rank residual continuation uncertainty without additional conditional rollouts. For a prompt \(x\), let \(p_0(x)\) denote the empirical success rate of the rollout group generated from the original prompt, and let \(R(\tau)\in\{0,1\}\) denote the terminal reward of trajectory \(\tau\). For \(s_i\) being the \(i\)-th state along a trajectory with length $|\tau|$, we define its normalized position as $\alpha_i = i/|\tau|$ .

VSR then estimates the continuation success probability at \(s_i\) by interpolating between the group-level success rate at the initial state and the realized terminal outcome:
\[
\hat p_t(s_i\mid \tau)=(1-\alpha_i)p_0(x)+\alpha_i R(\tau),
\qquad
\hat\sigma_t(s_i\mid \tau)=
\sqrt{\hat p_t(s_i\mid \tau)\bigl(1-\hat p_t(s_i\mid \tau)\bigr)}.
\]
Early prefixes are therefore scored mostly by the prompt-level rollout statistics, while late prefixes are scored mostly by the observed outcome of the trajectory they belong to.  Because the buffer stores only successful trajectories, $R(\tau)=1$ for every stored prefix, so the implemented score is determined by the prompt-level success rate and normalized prefix position. We therefore use \(\hat\sigma_t\) only as a zero-additional-rollout uncertainty-ranking proxy, not as a calibrated, unbiased, or consistent estimator of \(\sigma_{\pi_t}(s)\). Appendix~\ref{app:rebuttal-controls} reports a statewise validation showing moderate positive association with empirical continuation uncertainty, including within fixed prefix positions.

\paragraph{Replay update.}
At each update, VSR forms a mixed batch of fresh prompts and replayed prefixes. A replay ratio \(\rho\in[0,1]\) determines the fraction of batch elements drawn from \(\mathcal{B}_t\), with \(\rho=0\) recovering the base GRPO procedure. Replayed prefixes are sampled with probability proportional to \(\hat\sigma_t\). For each selected prefix, the current policy generates a fresh group of continuations; the prefix itself is treated as fixed context and masked out of the policy-gradient loss. The GRPO-style objective and group-relative advantage computation are otherwise unchanged.


\paragraph{Cold start.}
When the base policy almost never reaches successful trajectories, the replay buffer may be initialized with prefixes from a small set of expert successful trajectories. These prefixes are used only as reset states for fresh rollouts under \(\pi_t\); their tokens are masked from the loss exactly like ordinary replayed prefixes. Recency eviction then gradually replaces seeded prefixes with states produced by the current policy, so the proposal transitions from expert-supported states to on-policy states as training progresses.

Overall, the recent prefix buffer supplies the \(d^\pi\) factor by keeping replay states close to the current policy’s visitation distribution, while variance-weighted sampling supplies the \(\sigma_\pi\) factor by concentrating rollout budget on uncertain states. Successful-prefix filtering and cold-start seeding bias the construction toward states that lie on completable trajectories, and recency eviction lets the resulting proposal move with the policy. The individual components are not independent replay heuristics, but practical approximations to the same variance-driven state allocation rule.
\section{Experiments}
\label{sec:experiments}
Our experiments target the regime where a variance-driven proposal is predicted to matter most: tasks on which the base policy almost always fails from the initial state, but where some mid-trajectory states admit continuations with genuinely mixed outcomes. In this regime, $\sigma_\pi$ is concentrated on intermediate states that on-policy rollouts rarely reach. We accordingly evaluate VSR in two settings where signals are rare: mathematical reasoning on a hard subset of OpenR1-Math-220k (filtered to prompts with low base-policy pass rate), and long-horizon agentic decision making on Sokoban and difficult BabyAI levels, where successful trajectories remain sparse under the base policies. 

\subsection{Mathematical reasoning}

\paragraph{Setup.}
We train Qwen2.5-7B-Instruct \citep{qwen2} and Llama-3.2-3B-Instruct \citep{meta2024introducing} on a 1K hard subset of OpenR1-Math-220k \citep{openr1}. The subset is constructed by filtering for prompts with low base-policy pass rate; for Qwen2.5-7B-Instruct, we retain prompts with at most 1 success out of 32 samples. VSR is initialized with replay prefixes extracted from successful DeepSeek-R1 trajectories \citep{deepseekai2025deepseekr1incentivizingreasoningcapability}. These expert-seeded prefixes are used only as starting states for fresh rollouts, with prefix tokens masked from the policy-gradient loss, and are gradually replaced by recent on-policy prefixes through the replay buffer.

We compare VSR+GRPO against standard GRPO under the same training setup, and against SFT on the same expert trajectories used for cold-start initialization. We evaluate on AIME24 \citep{aime24}, AIME25 \citep{aime25}, AMC23, OlympiadBench \citep{he2024olympiadbench}, MATH500 \citep{hendrycks2021measuring}, and Minerva \citep{lewkowycz2022solving}. Final benchmark accuracy is reported in Table~\ref{tab:math}, and training dynamics, including reward, validation performance, and response length, are shown in Figure~\ref{fig:math-curve}.

\begin{table}[t]
\centering
\caption{Mathematical reasoning results (\%) on AIME24, AIME25, AMC23, OlympiadBench, MATH500, and Minerva. Reinforce-Ada-Est and static uniform expert-prefix resets are additional single-run controls on Llama-3.2-3B-Instruct. Best result in each column is in \textbf{bold}.}
\label{tab:math}
\footnotesize
\begin{tabular*}{\textwidth}{@{\extracolsep{\fill}}ccccccccc}
\toprule
\textbf{Model} & \textbf{Method} & \textbf{AIME24} & \textbf{AIME25} & \textbf{AMC23} & \textbf{Olympiad} & \textbf{MATH500} & \textbf{Minerva} & \textbf{Avg} \\
\midrule
\multirow{3}{*}{\textit{\shortstack{Qwen2.5-7B\\-Instruct}}}
& Base model   & $10.0$ & $7.9$ & $53.2$ & $\mathbf{39.2}$ & $\mathbf{75.6}$ & $37.3$ & $37.2$ \\
& GRPO         & $8.8$ & $4.8$ & $53.4$ & $38.3$ & $72.5$ & $38.6$ & $36.1$ \\
& VSR+GRPO     & $\mathbf{11.4}$ & $\mathbf{10.4}$ & $\mathbf{57.7}$ & $33.5$ & $73.7$ & $\mathbf{38.7}$ & $\mathbf{37.6}$ \\
\midrule
\multirow{6}{*}{\textit{\shortstack{Llama-3.2\\-3B-Instruct}}}
& Base model   & $5.3$ & $0.6$ & $21.8$ & $12.6$ & $40.0$ & $13.8$ & $15.7$ \\
& SFT          & $0.3$ & $0.4$ & $7.2$ & $5.0$ & $23.1$ & $5.4$ & $6.9$ \\
& GRPO         & $4.4$ & $\mathbf{1.0}$ & $18.7$ & $12.6$ & $39.2$ & $16.5$ & $15.4$ \\
& Reinforce-Ada & $4.2$ & $0.1$ & $18.0$ & $12.4$ & $39.9$ & $15.4$ & $15.0$ \\
& \shortstack{Uniform expert-prefix} & $5.6$ & $0.5$ & $\mathbf{22.7}$ & $14.1$ & $43.7$ & $15.8$ & $17.1$ \\
& VSR+GRPO     & $\mathbf{6.2}$ & $0.6$ & $20.4$ & $\mathbf{15.3}$ & $\mathbf{54.0}$ & $\mathbf{17.8}$ & $\mathbf{19.1}$ \\
\bottomrule
\end{tabular*}
\end{table}

\paragraph{Additional controls.} We include two additional single-run controls for the Llama-3.2-3B math setting. Static uniform expert-prefix resets use the same successful expert-trajectory source as VSR but sample uniformly from a fixed expert-prefix pool, probing how much of the gain comes from privileged reset states alone. Reinforce-Ada-Est\citep{xiong2025reinforce} reallocates a matched response budget across the original prompts without intermediate-state resets, testing whether prompt-level adaptive allocation alone reproduces the VSR gain. As shown in Table~\ref{tab:math}, static expert-prefix resets improve the average score from 15.4 for GRPO to 17.1, while Reinforce-Ada-Est obtains 15.0 and VSR reaches 19.1. Thus privileged resets explain part of the improvement, whereas prompt-level allocation alone does not reproduce the full VSR result. Detailed setups and limitations are provided in Appendix~\ref{app:rebuttal-controls}.

\begin{figure}[t]
\centering
\includegraphics[width=1.0\linewidth]{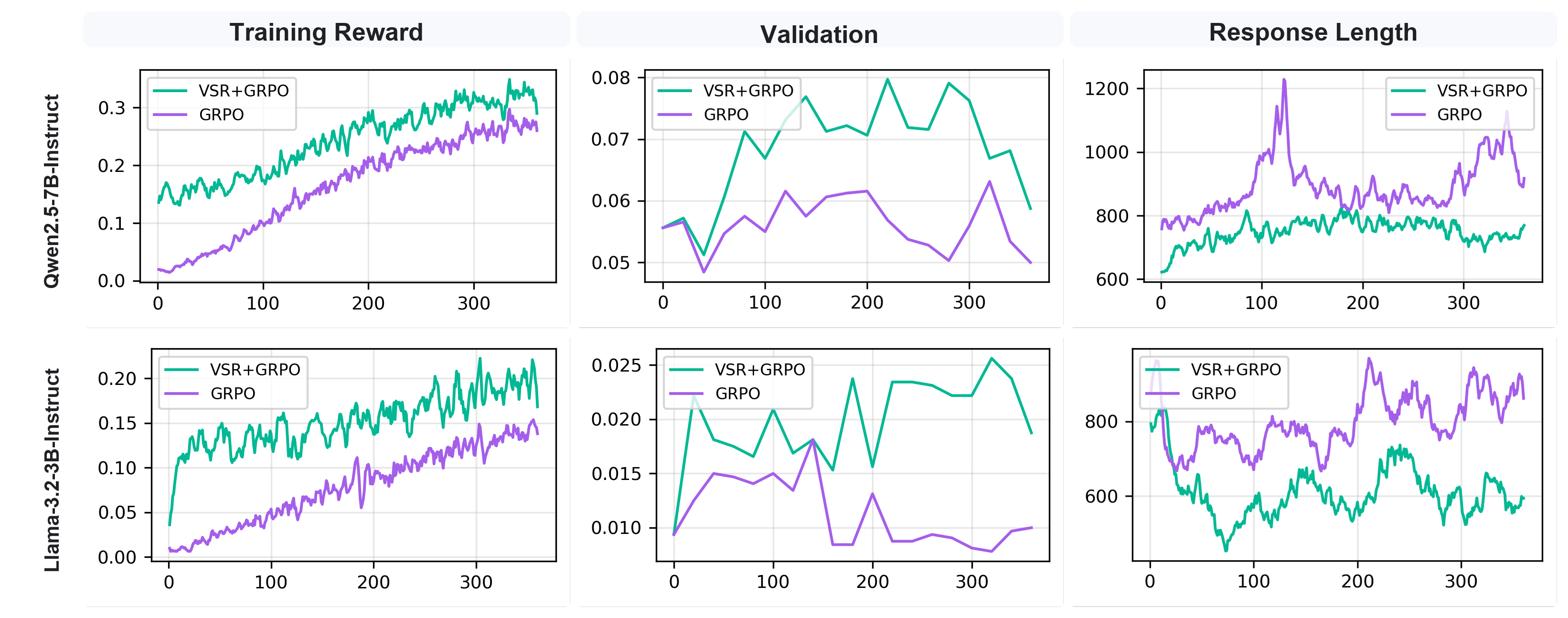}
\caption{Training dynamics on the 1K hard subset of OpenR1-Math-220k for Qwen2.5-7B-Instruct (top) and Llama-3.2-3B-Instruct (bottom). Columns show training reward, validation accuracy, and response length over training steps. VSR+GRPO sustains higher training reward and validation accuracy than GRPO with shorter responses.}
\label{fig:math-curve}
\end{figure}

\subsection{Long-horizon agentic tasks}
\label{sec:agentic}

\paragraph{Setup.}
We implement VSR on Sokoban\citep{gym_sokoban} and BabyAI\citep{chevalier2018babyai} using Qwen2.5-1.5B-Instruct and Qwen2.5-3B-Instruct\citep{qwen2}. For BabyAI, we use SynthLoc, OneRoom, and MiniBoss, which require multi-step planning and compositional decision making. We choose GiGPO \citep{feng2025group} as our base RL algorithm, and VSR is added as a plug-in state-resampling component. Unlike the math setting, the replay buffer is populated entirely from the model's own trajectories, without expert cold-start initialization.

We compare GiGPO+VSR against GiGPO under matched model, environment, and training settings. Final task success rates are reported in Table~\ref{tab:agentic}, and training dynamics, including reward, validation success, and newly generated continuation length after the selected reset state, are shown in Figure~\ref{fig:agent-curve}. The main training curves and most entries in Table~\ref{tab:agentic} are from single training runs; evaluation success rates use 128 fixed map seeds per environment, while independent-training-run uncertainty is reported separately for the scoring ablation below.

\begin{table}[t]
\centering
\caption{Agentic-task success rates (\%) on Sokoban and BabyAI. Best result is in \textbf{bold}.}
\label{tab:agentic}
\small
\begin{tabular}{ccccccc}
\toprule
& & \multirow{2}{*}{\textbf{Sokoban}} & \multicolumn{4}{c}{\textbf{BabyAI}} \\
\cmidrule(lr){4-7}
\textbf{Model} & \textbf{Method} & & \textbf{SynthLoc} & \textbf{OneRoom} & \textbf{MiniBoss} & \textbf{Avg} \\
\midrule
\multirow{3}{*}{\textit{Qwen2.5-1.5B-Instruct}}
& Base model    & $9.4$ & $11.6$ & $9.3$ & $21.4$ & $14.1$ \\
& GiGPO         & $12.7$ & $32.3$ & $27.9$ & $30.9$ & $30.4$ \\
& VSR+GiGPO     & $\mathbf{18.8}$ & $\mathbf{51.2}$ & $\mathbf{54.6}$ & $\mathbf{36.9}$ & $\mathbf{47.6}$ \\
\midrule
\multirow{3}{*}{\textit{Qwen2.5-3B-Instruct}}
& Base model    & $17.2$ & $9.3$ & $18.6$ & $16.7$ & $14.9$ \\
& GiGPO         & $21.5$ & $51.2$ & $46.5$ & $30.9$ & $42.9$ \\
& VSR+GiGPO     & $\mathbf{25.9}$ & $\mathbf{53.5}$ & $\mathbf{60.4}$ & $\mathbf{33.3}$ & $\mathbf{49.1}$ \\
\bottomrule
\end{tabular}
\end{table}

\begin{figure}[h]
\centering
\includegraphics[width=0.8\linewidth]{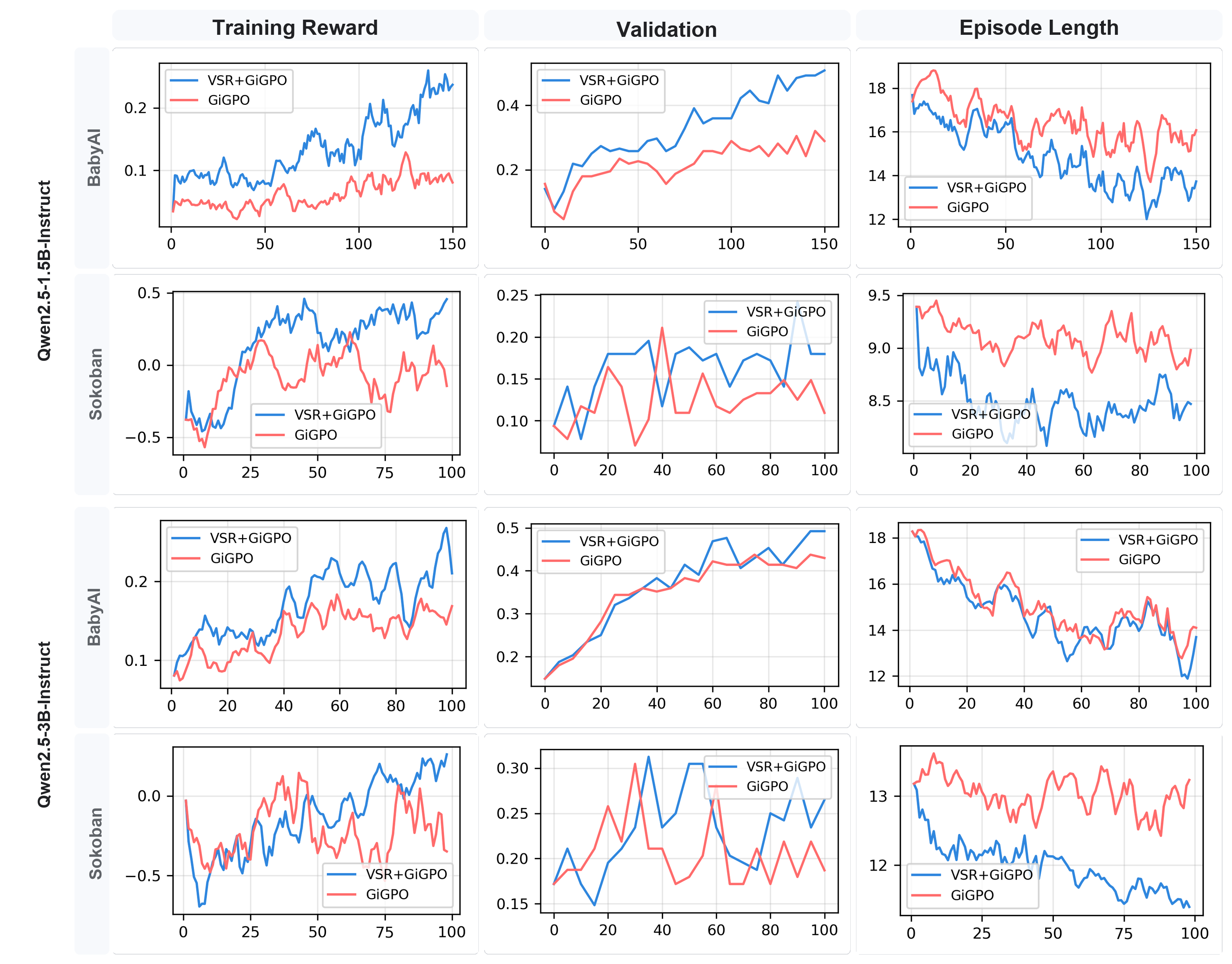}
\caption{Training dynamics on BabyAI and Sokoban for Qwen2.5-1.5B-Instruct (top two rows) and Qwen2.5-3B-Instruct (bottom two rows). Columns show training reward, validation success rate, and newly generated continuation length after the selected reset state over training steps. The shorter continuation length partly follows mechanically from starting a fraction of rollouts at intermediate states and should not be read as shorter complete trajectories from the original task state.}
\label{fig:agent-curve}
\end{figure}

\subsection{Ablations}
\label{sec:ablations}

We isolate the two engineering components of VSR through controlled variants, both built on top of GiGPO and evaluated on BabyAI under the same setup as Section~\ref{sec:agentic}. \textbf{(i) Uniform vs. variance-weighted sampling over the replay buffer.} VSR(w/o scoring) retains the recency-updated replay buffer but samples prefixes uniformly rather than in proportion to $\hat{\sigma}_t$. This isolates the contribution of the variance-driven weighting from that of prefix replay itself. \textbf{(ii) Static vs. policy-tracking replay buffer.} VSR(w/o recency) retains $\hat{\sigma}_t$-weighted sampling but disables the recency-based buffer update, so the buffer is populated once at the start of training and held fixed thereafter. This isolates the contribution of tracking $d^{\pi_t}$ from that of variance-driven weighting.

We additionally repeat the Qwen2.5-1.5B VSR-versus-uniform-replay scoring comparison across three independent training runs per variant. VSR reaches $48.2\pm2.4$ versus $44.0\pm4.3$ for uniform replay (mean $\pm$ SD), providing additional evidence for the benefit of variance-driven prioritization; full per-run results are reported in Appendix~\ref{app:rebuttal-controls}. The original no-recency curve remains a single run.

\begin{figure}[h]
\centering

\begin{minipage}[t]{0.48\linewidth}
\centering
\vspace{0pt}
\includegraphics[width=\linewidth]{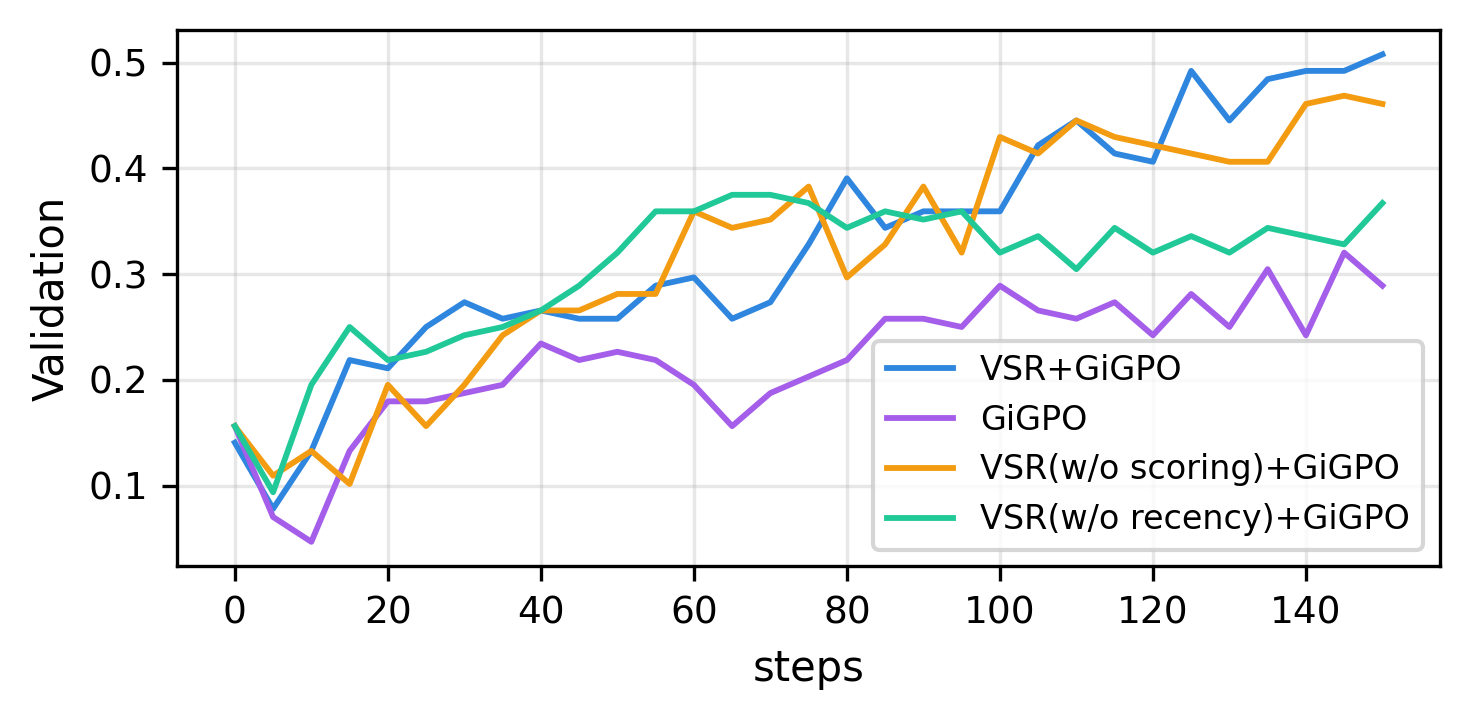}
\caption{Ablation on BabyAI: validation success rate over training.}
\label{fig:ablation}
\end{minipage}
\hfill
\begin{minipage}[t]{0.4\linewidth}
\centering
\vspace{0pt}
\vspace{0.1em} 

\captionof{table}{Ablation on BabyAI (avg success rate, \%). Best in \textbf{bold}.}
\label{tab:ablation}
\small
\begin{tabular}{lc}
\toprule
\textbf{Method} & \textbf{Avg} \\
\midrule
GiGPO              & $30.4$ \\
VSR(w/o scoring)   & $46.1$ \\
VSR(w/o recency)   & $36.7$ \\
VSR                & $\mathbf{47.6}$ \\
\bottomrule
\end{tabular}
\end{minipage}

\end{figure}

\subsection{Discussion}



\paragraph{VSR yields stronger per-batch training signal.} The training-reward curves in Figures~\ref{fig:math-curve} and~\ref{fig:agent-curve} indicate that VSR yields a higher per-batch gradient signal: VSR consistently gains more training reward over the whole training. On the hard math subset, GRPO's training reward continues to climb while its validation accuracy stays essentially flat, and final benchmark performance ends up below the base model on both models. The reason is that training signal is concentrated on a narrow handful of prompts that occasionally produce mixed-outcome groups, and on-policy sampling repeatedly draws gradient signal from that same small set. VSR addresses this at the level predicted by $\mu^\pi_{\mathrm{var}}$: even when the prompt-level support of $\sigma_\pi$ is narrow, its state-level support is much wider. 

\paragraph{VSR reduces the rollout cost of each policy update.} Because a non-trivial fraction of rollouts are initiated from mid-trajectory prefixes rather than from the original prompt, the expected number of newly generated tokens per trajectory is reduced; this is reflected in the consistently shorter newly generated continuation lengths in the agentic experiments and the analogous reduction in completion length observed in the mathematical reasoning experiments. Direct token accounting further supports this effect: at matched update counts, VSR+GiGPO generates 39.27M versus 81.30M tokens for Qwen2.5-1.5B and 34.43M versus 36.26M for Qwen2.5-3B; full accounting is reported in Appendix~\ref{app:rebuttal-controls}.

\paragraph{Variance-driven scoring improves training efficiency over uniform replay.} The comparison between VSR and VSR(w/o scoring) in Figure~\ref{fig:ablation} indicates that weighting buffered prefixes by $\hat{\sigma}_t$ yields higher validation performance on average over uniform sampling from the same buffer, particularly in the later phase of training. Across three independent training runs, VSR achieves $48.2\pm2.4$ versus $44.0\pm4.3$ for uniform replay (mean $\pm$ SD), further supporting the benefit of variance-driven prioritization; see Appendix~\ref{app:rebuttal-controls} for the individual runs. This indicates that the variance-driven weighting is not redundant with the replay mechanism itself: prioritizing prefixes with larger residual outcome variance translates into a more efficient use of the resampling budget.

\paragraph{Recency-based buffer updates are necessary to sustain learning.} VSR(w/o recency) outperforms full VSR initially but plateaus thereafter. Without ongoing updates, the buffer's empirical state distribution drifts away from $d^{\pi_t}$ as the policy evolves, and the policy is optimized against a proposal that no longer reflects its current visitation. The $d^\pi$ factor in $\mu_{\mathrm{var}}^\pi$ is precisely what guards against this distributional shift, consistent with the prediction in Section~\ref{sec:proposal:shape} that the variance-optimal proposal re-localizes with $\pi$.

\section{Conclusion}
\label{sec:conclusion}

We recast the signal-sparsity problem in RL with verifiable rewards as a variance-reduction problem over the state sampling distribution, and showed that the general trace-variance optimum depends on the action-gradient norm; under the stated within-state and across-state homogeneity conditions it reduces to the clean Neyman form $\mu^\pi_\mathrm{var} \propto d^\pi(s)\sigma_\pi(s)$, whose proposal-corrected estimator has the same standardized-residual form as GRPO's group-relative normalization. Building on this allocation principle, we proposed VSR, a practical prefix-replay method that improves over the compared GRPO/GiGPO baselines on most evaluated settings.

A main limitation is the gap between the exact proposal-corrected estimator and practical VSR. VSR does not estimate the exact density ratio, stores only successful prefixes, uses expert-prefix cold start in the math experiments, and relies on a coarse zero-additional-rollout uncertainty score. The additional controls and validations in Appendix~\ref{app:rebuttal-controls} clarify these effects but do not completely isolate every practical component; most main agentic results and the no-recency ablation also remain single-run. Tighter low-cost state scores and broader multi-seed evaluation remain important directions.

\bibliographystyle{plainnat}  
\bibliography{references}  


\appendix

\section{Related Work}

A broad line of recent work on sparse-reward RL for LLMs aims, at varying levels of explicitness, to address the same underlying difficulty: when terminal reward is binary and success is rare, naively sampling rollouts under the current policy yields many uninformative updates. We organize the discussion around what each line of work effectively reshapes—the prompt distribution, the within-trajectory state at which generation is conditioned, the task itself, or the per-prompt sampling budget.

\textbf{Selecting and scheduling existing prompts.} A first family of methods operates at the prompt level. Filtering and dynamic sampling discard prompts whose rollout groups produce uniform reward, on the grounds that they contribute no group-relative signal \citep{yu2025dapo,cui2025process,wang2025ragen,gulcehre2023reinforced}. Curriculum-style schedulers go further and order prompts by difficulty or learnability, ranging from classical curriculum learning \citep{bengio2009curriculum,portelas2020automatic} to recent LLM-specific instantiations that adapt the prompt distribution online \citep{zhao2024automatic,shi2025efficient,gao2025prompt,team2025kimi,setlur2025e3}. RLVE \citep{zeng2025rlve} makes the environment itself adaptive, generating verifiable problem instances calibrated to the current learner. A separate but related line studies the role of the sampler in preference-based post-training, where sampling distribution similarly affects gradient signal \citep{shi2024crucial,feng2025pilaf}. The granularity of intervention in this family is the prompt: once a prompt is selected, the distribution of intermediate states encountered along its rollouts is left entirely to the current policy.

\textbf{Conditioning on intermediate states within a trajectory.} A second family conditions generation on intermediate states rather than only on prompts. Prefix-conditioned RL methods inject partial trajectories—either expert-derived or drawn from earlier policy iterates—as starting points for fresh rollouts \citep{zhang2025bread,qu2026pope,yi2026pivotrl}, extending an older RL tradition of reverse curricula style state resets \citep{andrychowicz2017hindsight,florensa2017reverse,ecoffet2019go,xi2024training} . Hint injection during rollout \citep{liao2026self} and recursive problem decomposition \citep{simonds2025ladder} can be viewed as variants of the same idea, supplying privileged context that places the policy at a more learnable state before continuation. ReST \citep{gulcehre2023reinforced} is a closer precedent on the replay side, iteratively training on successful trajectories, though as an offline SFT-then-RL procedure rather than as an online state-resampling mechanism. The criterion by which intermediate states are selected varies across these methods—difficulty, expert provenance, or a hand-designed schedule.

\textbf{Generating new tasks: self-play schemes.} A third family sidesteps the question of which existing state to train on by generating new tasks calibrated to the learner. Self-play and proposer–solver frameworks have the model—or a paired proposer model—propose problems that are subsequently solved and verified, with the proposal policy itself trained to produce informative tasks \citep{haluptzok2022language,zhao2025absolute,huang2025r,chen2025self,lu2025search,liu2025spiral,wilf2025propose,wang2026rlanything} . The motivation of self-play schemes overlaps with curriculum learning and prefix conditioning, putting training mass on states from which learning signal is accessible

\textbf{Variance-driven sampling budget reallocation.} Closest in motivation are methods that explicitly target estimator variance.  While a classical thread on variance reduction in policy-gradient estimators more broadly \citep{greensmith2004variance} focuses on baseline design rather than proposal design, some recent researches focus on prompt sampling. Reinforce-Ada \citep{xiong2025reinforce} adaptively allocates rollouts across prompts under non-linear RL objectives, and \citet{yao2025optimizing} minimizes gradient variance through rejection sampling and rollout reallocation in chain-of-thought RL. Both are concerned with the same underlying inefficiency—sampling budget spent on prompts that yield no usable signal—and both reallocate rollouts across prompts, holding the per-prompt distribution of intermediate states fixed. 

\section{Theoretical analysis of the variance-driven proposal}
\label{app:theory}

This appendix collects the full theoretical analysis underlying the
variance-driven proposal $\mu_{\mathrm{var}}\propto d^\pi(s)\sigma_\pi(s)$
introduced in Section~\ref{sec:proposal}. Section~\ref{app:setup} establishes
notation and the working form of the estimator variance.
Section~\ref{app:exact} derives the exact variance-optimal proposal.
Section~\ref{app:clean} specializes it to the standard gradient-uniformity
convention, recovering $\mu_{\mathrm{var}}$ and the GRPO-style estimator.
Section~\ref{app:weak} relaxes that convention and gives a quantitative
condition---weaker than gradient uniformity---under which $\mu_{\mathrm{var}}$
still strictly reduces variance relative to on-policy sampling.

\subsection{Setup}
\label{app:setup}

Throughout, we use the state-action view of LLM RL described in
Section~\ref{sec:setup}: $s_t=(x,y_{<t})$, $a_t=y_t$, $s_0=x\sim\mathcal{D}$,
with terminal verifiable reward $R\in\{0,1\}$. The state-visitation
distribution is
\[
d^\pi(s) \;\propto\; \sum_{t=0}^{T-1}\Pr\!\big[s_t=s\,\big|\,s_0\sim\mathcal{D},\,\pi\big].
\]
For a state-dependent baseline $b$ and a proposal $\mu$ with
$\mathrm{supp}(\mu)\supseteq\mathrm{supp}(d^\pi)$, the proposal-corrected
policy-gradient estimator is
\begin{equation}
\label{eq:app-pg}
\widehat{\nabla J}_{b,\mu}
  \;=\; \frac{d^\pi(s)}{\mu(s)}\big(r(s,a)-b(s)\big)\nabla\log\pi(a\mid s),
  \qquad s\sim\mu,\ a\sim\pi(\cdot\mid s),
\end{equation}
which is unbiased for $\nabla J(\pi)$ for every admissible pair $(b,\mu)$.
We write $g(s,a):=\nabla\log\pi(a\mid s)$, $\delta(s,a):=r(s,a)-b(s)$,
$\bar r(s):=\mathbb{E}_{a\sim\pi(\cdot\mid s)}[r(s,a)]$, and use $\sigma_\pi(s)$
for the within-state reward standard deviation,
$\sigma_\pi(s)^2:=\mathrm{Var}_{a\sim\pi(\cdot\mid s)}\!\big(r(s,a)\big)$.

For a fixed baseline $b$, define the gradient-signal second moment
\begin{equation}
\label{eq:M-def}
M_b(s) \;:=\; \mathbb{E}_{a\sim\pi(\cdot\mid s)}\!\big[(r(s,a)-b(s))^2\,H(s,a)\big],
\qquad H(s,a):=\|g(s,a)\|^2.
\end{equation}

\begin{lemma}[Working form of the variance]
\label{lem:V-form}
For any admissible $(b,\mu)$,
\begin{equation}
\label{eq:V-form}
V(\mu;b) \;:=\; \mathrm{tr}\,\mathrm{Cov}\!\big(\widehat{\nabla J}_{b,\mu}\big)
  \;=\; \mathbb{E}_{s\sim d^\pi}\!\left[\frac{d^\pi(s)}{\mu(s)}\,M_b(s)\right]
       \;-\;\|\nabla J(\pi)\|^2.
\end{equation}
\end{lemma}

\begin{proof}
By unbiasedness, $\mathrm{tr}\,\mathrm{Cov}(\widehat{\nabla J}_{b,\mu})=\mathbb{E}\,\|\widehat{\nabla J}_{b,\mu}\|^2-\|\nabla J(\pi)\|^2$. Expanding $s\sim\mu$ gives
\[
\mathbb{E}\,\|\widehat{\nabla J}_{b,\mu}\|^2
= \mathbb{E}_{d^\pi}\!\left[\frac{d^\pi(s)}{\mu(s)}M_b(s)\right]. \qedhere
\]
\end{proof}

\subsection{The exact joint optimum}
\label{app:exact}

\begin{theorem}[Exact joint-optimal baseline and proposal]
\label{thm:exact}
For each state, the trace-variance-minimizing baseline is
\begin{equation}
\label{eq:exact-baseline}
b_w^*(s)=\frac{\mathbb{E}[r(s,a)H(s,a)\mid s]}{\mathbb{E}[H(s,a)\mid s]}.
\end{equation}
Defining
\begin{equation}
\label{eq:M-star-def}
M_w^*(s):=\mathbb{E}[(r(s,a)-b_w^*(s))^2H(s,a)\mid s],
\end{equation}
the joint-optimal proposal on the effective support is
\begin{equation}
\label{eq:exact-opt}
\mu^{\pi,\star}(s)\propto d^\pi(s)\sqrt{M_w^*(s)}.
\end{equation}
If full support is required, the proposal can be smoothed by mixing with $d^\pi$.
\end{theorem}

\begin{proof}
For fixed $\mu$, the statewise factor multiplying $M_b(s)$ in~\eqref{eq:V-form} is positive, so minimizing over $b(s)$ is equivalent to minimizing $\mathbb{E}[(r-b)^2H\mid s]$. Differentiating in $b$ gives~\eqref{eq:exact-baseline}. Substituting $b_w^*$ into~\eqref{eq:V-form}, Cauchy--Schwarz yields the minimizing proposal $\mu(s)\propto d^\pi(s)\sqrt{M_w^*(s)}$ exactly as in the standard Neyman-allocation argument.
\end{proof}

\subsection{Clean case: conditions yielding $d^\pi\sigma_\pi$}
\label{app:clean}

The exact joint optimum couples reward variability and action-gradient norms. The following two conditions separate these quantities.

\begin{assumption}[Within-state conditional-moment condition]
\label{ass:G-within}
For each state $s$, there exists $G(s)>0$ such that
\begin{equation}
\mathbb{E}[H(s,a)\mid s,r]=G(s).
\end{equation}
That is, conditional on the state, the expected squared score-function norm does not depend on the realized reward.
\end{assumption}

Under Assumption~\ref{ass:G-within}, $b_w^*(s)=\bar r(s)$ and
\begin{equation}
M_w^*(s)=G(s)\sigma_\pi(s)^2,
\qquad
\mu^{\pi,\star}(s)\propto d^\pi(s)\sigma_\pi(s)\sqrt{G(s)}.
\end{equation}

\begin{assumption}[Across-state gradient-scale homogeneity]
\label{ass:G}
$G(s)\equiv G_0$ on the effective support.
\end{assumption}

\begin{theorem}[Clean Neyman form]
\label{thm:clean}
Under Assumptions~\ref{ass:G-within} and~\ref{ass:G}, the joint-optimal proposal on the effective support reduces to
\begin{equation}
\label{eq:mu-var-form}
\mu^{\pi,\star}(s) \;\propto\; d^\pi(s)\,\sigma_\pi(s) \;=:\; \mu_{\mathrm{var}}(s),
\end{equation}
and the exact proposal-corrected estimator is proportional to
\begin{equation}
\label{eq:joint-opt}
\frac{r(s,a)-\bar r(s)}{\sigma_\pi(s)}\nabla\log\pi(a\mid s),
\qquad s\sim\mu_{\mathrm{var}},\ a\sim\pi(\cdot\mid s).
\end{equation}
Moreover,
\begin{equation}
\label{eq:clean-gap}
V(d^\pi;\bar r)-V(\mu_{\mathrm{var}};\bar r)=G_0\,\mathrm{Var}_{s\sim d^\pi}(\sigma_\pi(s))\ge 0,
\end{equation}
with strict inequality whenever $\sigma_\pi$ is non-constant on the effective support.
\end{theorem}

\begin{proof}
Assumption~\ref{ass:G-within} gives $b_w^*=\bar r$ and $M_w^*=G(s)\sigma_\pi^2(s)$. Assumption~\ref{ass:G} then gives $M_w^*=G_0\sigma_\pi^2$, so~\eqref{eq:exact-opt} reduces to~\eqref{eq:mu-var-form}; substituting in the proposal-corrected estimator yields~\eqref{eq:joint-opt}. The variance-gap expression follows from the same Jensen calculation as in the original clean-case derivation.
\end{proof}

\subsection{Weak version: variance reduction beyond gradient uniformity}
\label{app:weak}

The clean conditions above are idealizations that fail to varying degrees in
practice: different states have different conditional gradient-norm scales,
and these scales may co-vary with $\sigma_\pi$ across states. The natural
question is whether $\mu_{\mathrm{var}}\propto d^\pi\sigma_\pi$---which no
longer coincides with the exact optimum~\eqref{eq:exact-opt} once
Assumption~\ref{ass:G} is dropped---is still a strict improvement over
on-policy sampling, and under what condition.

To state this cleanly we factor $M(s)$ through $\sigma_\pi(s)$. On
$\{s:\sigma_\pi(s)>0\}$, define
\begin{equation}
\label{eq:K-def}
K(s) \;:=\; \frac{M(s)}{\sigma_\pi(s)^2}.
\end{equation}
Assumption~\ref{ass:G} is equivalent to $K\equiv G_0$ being a state-independent
constant; relaxing it amounts to allowing $K$ to vary across $s$.

For notational simplicity, we assume throughout this subsection that
$\sigma_\pi(s)>0$ for $d^\pi$-almost every state. States with
$\sigma_\pi(s)=0$ carry no reward-variance signal under $b=\bar r$ and can be
removed by working on the effective support, with the corresponding
normalization factors tracked explicitly.

The relaxed assumption is a quantitative bound on the heterogeneity of $K$,
measured by its coefficient of variation, relative to the spread of
$\sigma_\pi$:

\begin{assumption}[Gradient regularity]
\label{ass:G-prime}
On the set of states $\{s:\sigma_\pi(s)>0\}$, the function $K(s)$ satisfies
\begin{equation}
\label{eq:G-prime}
\frac{\sqrt{\mathrm{Var}_{d^\pi}(K)}}{\mathbb{E}_{d^\pi}[K]}
  \;<\; \frac{2}{3}\sqrt{\mathrm{Var}_{d^\pi}(\sigma_\pi)},
\end{equation}
where the moments are taken with respect to $s\sim d^\pi$ restricted to
that set.
\end{assumption}

Assumption~\ref{ass:G-prime} is strictly weaker than Assumption~\ref{ass:G}:
gradient uniformity is the limit $\mathrm{Var}(K)=0$, in which the left-hand
side is zero and~\eqref{eq:G-prime} is vacuous (provided $\sigma_\pi$ is
non-constant). It admits a transparent reading: $\mu_{\mathrm{var}}$ remains
beneficial as long as the relative state-to-state variation of $K$ is small
compared to the absolute spread of $\sigma_\pi$ on the policy's reachable
set---i.e., as long as the policy's reachable-but-uncertain frontier is wide
enough to outweigh whatever heterogeneity $K$ exhibits across states.

\begin{theorem}[Variance reduction under gradient regularity]
\label{thm:weak}
Suppose $r\in\{0,1\}$ and $b=\bar r$. Under Assumption~\ref{ass:G-prime},
\begin{equation}
\label{eq:weak-gap}
V(\mu_{\mathrm{var}};b^*) \;<\; V(d^\pi;b^*).
\end{equation}
\end{theorem}

The proof proceeds via an exact decomposition of the variance gap, isolating
a non-positive Jensen-type term and a coupling correction that vanishes
under Assumption~\ref{ass:G}.

\begin{lemma}[Variance gap decomposition]
\label{lem:gap}
Let $b=\bar r$. Then, with all expectations, variances, and covariances taken
with respect to $s\sim d^\pi$ restricted to $\{\sigma_\pi>0\}$,
\begin{equation}
\label{eq:gap}
V(\mu_{\mathrm{var}};b^*) - V(d^\pi;b^*)
  \;=\; -\,\mathbb{E}[K]\cdot\mathrm{Var}(\sigma_\pi) \;+\; \Delta_{\mathrm{coup}},
\end{equation}
where $\Delta_{\mathrm{coup}} := \mathbb{E}[\sigma_\pi]\,\mathrm{Cov}(\sigma_\pi,K) - \mathrm{Cov}(\sigma_\pi^2,K)$.
\end{lemma}

\begin{proof}
Let $Z:=\mathbb{E}_{d^\pi}[\sigma_\pi]>0$, so $\mu_{\mathrm{var}}=d^\pi\sigma_\pi/Z$.
By Lemma~\ref{lem:V-form}, the $\|\nabla J\|^2$ terms cancel and
\[
V(\mu_{\mathrm{var}};b^*)-V(d^\pi;b^*)
= Z\,\mathbb{E}_{d^\pi}\!\left[\frac{M(s)}{\sigma_\pi(s)}\right] - \mathbb{E}_{d^\pi}[M(s)].
\]
Substituting $M=\sigma_\pi^2 K$,
\[
= \mathbb{E}[\sigma_\pi]\,\mathbb{E}[\sigma_\pi K] - \mathbb{E}[\sigma_\pi^2 K].
\]
Apply $\mathbb{E}[XY]=\mathbb{E}[X]\mathbb{E}[Y]+\mathrm{Cov}(X,Y)$ to both
products:
\begin{align*}
\mathbb{E}[\sigma_\pi]\,\mathbb{E}[\sigma_\pi K]
   &= \mathbb{E}[\sigma_\pi]^2\,\mathbb{E}[K] + \mathbb{E}[\sigma_\pi]\,\mathrm{Cov}(\sigma_\pi,K),\\
\mathbb{E}[\sigma_\pi^2 K]
   &= \mathbb{E}[\sigma_\pi^2]\,\mathbb{E}[K] + \mathrm{Cov}(\sigma_\pi^2,K).
\end{align*}
Subtracting and using
$\mathrm{Var}(\sigma_\pi)=\mathbb{E}[\sigma_\pi^2]-\mathbb{E}[\sigma_\pi]^2$
gives~\eqref{eq:gap}.
\end{proof}

\begin{lemma}[Variance-of-square bound]
\label{lem:var-square}
For any random variable $X$ with values in $[0,M]$,
$\mathrm{Var}(X^2)\le 4M^2\,\mathrm{Var}(X)$.
\end{lemma}

\begin{proof}
Let $X'$ be an independent copy of $X$. Then
\[
\mathrm{Var}(X^2)
  = \tfrac12\,\mathbb{E}[(X^2-X'^2)^2]
  = \tfrac12\,\mathbb{E}[(X+X')^2(X-X')^2]
  \le 2M^2\,\mathbb{E}[(X-X')^2]
  = 4M^2\,\mathrm{Var}(X),
\]
where the inequality uses $0\le X+X'\le 2M$ almost surely.
\end{proof}

\begin{proof}[Proof of Theorem~\ref{thm:weak}]
By Lemma~\ref{lem:gap}, it suffices to show
$\mathbb{E}[K]\,\mathrm{Var}(\sigma_\pi) > \Delta_{\mathrm{coup}}$. We bound
$|\Delta_{\mathrm{coup}}|$ from above.

By Cauchy--Schwarz,
\[
|\mathrm{Cov}(\sigma_\pi,K)|\le\sqrt{\mathrm{Var}(\sigma_\pi)\mathrm{Var}(K)},
\qquad
|\mathrm{Cov}(\sigma_\pi^2,K)|\le\sqrt{\mathrm{Var}(\sigma_\pi^2)\mathrm{Var}(K)}.
\]
For binary rewards, $\sigma_\pi(s)\le 1/2$ structurally, so
Lemma~\ref{lem:var-square} with $M=1/2$ gives
$\mathrm{Var}(\sigma_\pi^2)\le\mathrm{Var}(\sigma_\pi)$ and hence
$|\mathrm{Cov}(\sigma_\pi^2,K)|\le\sqrt{\mathrm{Var}(\sigma_\pi)\mathrm{Var}(K)}$.
Combined with $\mathbb{E}[\sigma_\pi]\le 1/2$,
\begin{equation}
\label{eq:coup-bound}
|\Delta_{\mathrm{coup}}|
\;\le\; \tfrac12\sqrt{\mathrm{Var}(\sigma_\pi)\mathrm{Var}(K)}
       + \sqrt{\mathrm{Var}(\sigma_\pi)\mathrm{Var}(K)}
\;=\; \tfrac32\sqrt{\mathrm{Var}(\sigma_\pi)\mathrm{Var}(K)}.
\end{equation}
Plugging~\eqref{eq:coup-bound} into~\eqref{eq:gap},
\[
V(\mu_{\mathrm{var}};b^*) - V(d^\pi;b^*)
\;\le\; -\,\mathbb{E}[K]\,\mathrm{Var}(\sigma_\pi) + \tfrac32\sqrt{\mathrm{Var}(\sigma_\pi)\mathrm{Var}(K)}.
\]
The right-hand side is negative iff
$\mathbb{E}[K]\sqrt{\mathrm{Var}(\sigma_\pi)} > \tfrac32\sqrt{\mathrm{Var}(K)}$,
i.e.\ iff $\sqrt{\mathrm{Var}(K)}/\mathbb{E}[K] < \tfrac23\sqrt{\mathrm{Var}(\sigma_\pi)}$,
which is exactly Assumption~\ref{ass:G-prime}.
\end{proof}

\begin{remark}[Beyond binary rewards]
\label{rem:dense}
For $r\in[0,R]$, $\sigma_\pi(s)\le R/2$ and Lemma~\ref{lem:var-square} gives
$\mathrm{Var}(\sigma_\pi^2)\le R^2\mathrm{Var}(\sigma_\pi)$. The proof of
Theorem~\ref{thm:weak} carries over with constants scaled accordingly:
$|\Delta_{\mathrm{coup}}|\le\frac{3R}{2}\sqrt{\mathrm{Var}(\sigma_\pi)\mathrm{Var}(K)}$,
and Assumption~\ref{ass:G-prime} becomes
$\sqrt{\mathrm{Var}(K)}/\mathbb{E}[K]<\tfrac{2}{3R}\sqrt{\mathrm{Var}(\sigma_\pi)}$.
\end{remark}

\section{State-Visitation Policy Mirror Descent: Theory}
\label{app:svmd}

This appendix analyzes an idealized algorithm---State-Visitation Policy Mirror
Descent (SVMD)---that samples states from a smoothed version of $d^\pi\sigma$
and uses an explicit critic fitted by importance-weighted regression. SVMD is
\emph{not} the GRPO-based algorithm of the main text: it assumes an oracle for
$\sigma_h^\pi(s)$ and replaces GRPO's group-relative advantage with a fitted
$Q$. We include it as theoretical evidence that proposals of the form
$d^\pi\sigma$ are well-behaved in a tractable PMD setting; a direct
finite-sample analysis of the practical algorithm is left to future work.

\subsection{Setup}

We consider a finite-horizon episodic MDP with horizon $H$, state space
$\mathcal{S}$, action space $\mathcal{A}$, transition kernels
$P_h(\cdot\mid s,a)$, reward functions $r_h(s,a)\in[0,1]$, and initial state
distribution $\rho$. A (non-stationary) policy is
$\pi=(\pi_1,\dots,\pi_H)$. Define
\[
V_h^\pi(s):=\mathbb{E}_\pi\!\left[\sum_{\ell=h}^H r_\ell(s_\ell,a_\ell)\,\Big|\, s_h=s\right],\qquad
Q_h^\pi(s,a):=r_h(s,a)+\mathbb{E}_{s'\sim P_h(\cdot\mid s,a)}[V_{h+1}^\pi(s')],
\]
\[
A_h^\pi(s,a):=Q_h^\pi(s,a)-V_h^\pi(s),\qquad
J(\pi):=\mathbb{E}_{s_1\sim\rho}[V_1^\pi(s_1)].
\]
Let $d_h^\pi$ denote the stage-$h$ state occupancy of $\pi$, and let
$\nu_h^\pi(s,a):=d_h^\pi(s)\pi_h(a\mid s)$ be the stage-$h$ state-action
occupancy. Throughout the appendix, $\sigma_h^\pi:\mathcal{S}\to\mathbb{R}_{\geq 0}$
denotes a fixed nonnegative \emph{state-score function} associated with policy
$\pi$ at stage $h$ (e.g.\ $\sigma_h^\pi(s)=\sqrt{\mathrm{Var}_{a\sim\pi_h(\cdot\mid s)}(Q_h^\pi(s,a))}$).
Our analysis does not depend on the specific form of $\sigma$.

\subsection{Algorithm}

\begin{algorithm}[t]
\caption{State-Visitation Policy Mirror Descent (SVMD)}
\label{alg:svmd}
\begin{algorithmic}[1]
\Require Initial policy $\pi_1$, horizon $H$, iterations $T$, per-stage state
budgets $\{N_h\}_{h=1}^H$, step size $\eta>0$, smoothing parameter
$\varepsilon\in(0,1]$, state-score oracle $\sigma_h^\pi(\cdot)$, stagewise
critic classes $\{\mathcal{F}_h\}_{h=1}^H$.
\For{$t=1,\dots,T$}
    \For{$h=1,\dots,H$}
        \State Define the proposal
        \[
        \widetilde\mu_h^t(s)\propto d_h^{\pi_t}(s)\sigma_h^{\pi_t}(s),
        \qquad
        \mu_h^t(s):=(1-\varepsilon)\widetilde\mu_h^t(s)+\varepsilon\, d_h^{\pi_t}(s),
        \]
        and importance weights $w_{h}^t(s):=d_h^{\pi_t}(s)/\mu_h^t(s)\in[0,1/\varepsilon]$.
        \State Sample i.i.d.\ data $\{(s_{h,i},a_{h,i},G_{h,i})\}_{i=1}^{N_h}$:
        $s_{h,i}\sim\mu_h^t$, $a_{h,i}\sim\pi_{t,h}(\cdot\mid s_{h,i})$, and
        $G_{h,i}:=\sum_{\ell=h}^H r_\ell(s_\ell,a_\ell)$ is the Monte Carlo
        return-to-go obtained by rolling out $\pi_t$ from $(s_{h,i},a_{h,i})$.
        \State \emph{Weighted critic fit:}
        \[
        \widehat Q_h^t\in\arg\min_{f\in\mathcal{F}_h}
        \sum_{i=1}^{N_h} w_h^t(s_{h,i})\bigl(f(s_{h,i},a_{h,i})-G_{h,i}\bigr)^2.
        \]
        \State \emph{KL-regularized mirror ascent:} for each $s\in\mathcal{S}$,
        \[
        \pi_{t+1,h}(\cdot\mid s)\leftarrow
        \arg\max_{p\in\Delta(\mathcal{A})}
        \Bigl\{\eta\,\langle\widehat Q_h^t(s,\cdot),\,p\rangle
        -\mathrm{KL}\!\bigl(p\,\big\|\,\pi_{t,h}(\cdot\mid s)\bigr)\Bigr\}.
        \]
    \EndFor
\EndFor
\State \Return uniform mixture $\bar\pi_T:=\mathrm{Unif}\{\pi_1,\dots,\pi_T\}$.
\end{algorithmic}
\end{algorithm}

The mirror update has the closed form
$\pi_{t+1,h}(a\mid s)\propto\pi_{t,h}(a\mid s)\exp(\eta\widehat Q_h^t(s,a))$,
i.e.\ a softmax NPG step.

\subsection{Assumptions}

\begin{assumption}[Bounded rewards and bounded critic class]
\label{asm:bdd}
For all $h,s,a$, $r_h(s,a)\in[0,1]$, hence $0\leq Q_h^\pi(s,a)\leq H$ for all
$\pi,h,s,a$. Moreover $\mathcal{F}_h\subseteq\{f:\mathcal{S}\times\mathcal{A}\to[0,H]\}$.
\end{assumption}

\begin{assumption}[Stagewise realizability]
\label{asm:real}
For every iteration $t\in[T]$ and stage $h\in[H]$, $Q_h^{\pi_t}\in\mathcal{F}_h$.
\end{assumption}

\begin{assumption}[Comparator coverage]
\label{asm:cov}
There exist $\kappa\geq 1$ and a comparator policy $\mu$ such that for all
$t\in[T]$, $h\in[H]$, $(s,a)$,
\[
\nu_h^\mu(s,a)\leq \kappa\,\nu_h^{\pi_t}(s,a).
\]
\end{assumption}

\begin{assumption}[Proposal smoothing]
\label{asm:smooth}
For every $t\in[T]$, $h\in[H]$, $s\in\mathcal{S}$,
$\mu_h^t(s)\geq \varepsilon\, d_h^{\pi_t}(s)$, equivalently
$d_h^{\pi_t}(s)/\mu_h^t(s)\leq 1/\varepsilon$. (This holds by construction of
$\mu_h^t$ in Algorithm~\ref{alg:svmd}.)
\end{assumption}

\subsection{Main Theorem}

\begin{theorem}[Finite-sample guarantee for SVMD]
\label{thm:svmd}
Under Assumptions~\ref{asm:bdd}--\ref{asm:smooth}, fix $\delta\in(0,1)$. With
probability at least $1-\delta$, the policy mixture $\bar\pi_T$ returned by
Algorithm~\ref{alg:svmd} satisfies
\[
J(\mu)-J(\bar\pi_T)\;\leq\;
\frac{D_\mu}{\eta T}\;+\;\frac{\eta H^3}{2}\;+\;2\kappa\sum_{h=1}^H \beta_h,
\]
where
\[
D_\mu:=\sum_{h=1}^H \mathbb{E}_{s\sim d_h^\mu}\!\left[\mathrm{KL}\!\bigl(\mu_h(\cdot\mid s)\,\big\|\,\pi_{1,h}(\cdot\mid s)\bigr)\right],
\qquad
\beta_h:=\sqrt{\,C\cdot\frac{H^2}{\varepsilon}\cdot\frac{\log(HT|\mathcal{F}_h|/\delta)}{N_h}\,},
\]
for a universal constant $C>0$. In particular, choosing
$\eta=\sqrt{2D_\mu/(H^3 T)}$ yields
\[
J(\mu)-J(\bar\pi_T)\;\leq\;\sqrt{\frac{2D_\mu H^3}{T}}\;+\;2\kappa\sum_{h=1}^H \beta_h.
\]
\end{theorem}

\subsection{Proof Outline}

The proof has three modular parts.
\begin{enumerate}[leftmargin=1.5em]
\item \textbf{Per-step mirror descent inequality} (Lemma~\ref{lem:mirror}).
The KL-regularized update yields, stagewise and per iteration, a bound on
$\eta\langle\widehat Q_h^t,\mu_h-\pi_{t,h}\rangle$ in terms of the progress
$\mathrm{KL}(\mu_h\|\pi_{t,h})-\mathrm{KL}(\mu_h\|\pi_{t+1,h})$ plus a
quadratic drift term.
\item \textbf{Telescoping and averaging.} Summing the per-step inequality
over $t$ and $h$ and dividing by $T$ controls the average inner product
between $\widehat Q_h^t$ and $\mu_h-\pi_{t,h}$.
\item \textbf{Critic error control} (Lemma~\ref{lem:critic}).
The performance difference lemma reduces $J(\mu)-J(\pi_t)$ to the same inner
product plus a critic error $\mathcal{E}_{h,t}$. Coverage
(Assumption~\ref{asm:cov}) bounds $|\mathcal{E}_{h,t}|$ by an on-policy $L^1$
critic error, which is in turn controlled by weighted ERM analysis under
realizability and proposal smoothing.
\end{enumerate}

\subsection{Lemmas}

\begin{lemma}[Performance difference]
\label{lem:pdl}
For any policies $\pi,\mu$,
\[
J(\mu)-J(\pi)
=\sum_{h=1}^H\mathbb{E}_{s\sim d_h^\mu}\!\Bigl[\bigl\langle Q_h^\pi(s,\cdot),\,
\mu_h(\cdot\mid s)-\pi_h(\cdot\mid s)\bigr\rangle\Bigr].
\]
\end{lemma}
\begin{proof}
Standard finite-horizon performance difference lemma; see e.g.\ Kakade and
Langford (2002).
\end{proof}

\begin{lemma}[Mirror descent per-step inequality]
\label{lem:mirror}
Fix $t\in[T]$ and $h\in[H]$, and any state $s$. Let
$\pi_{t+1,h}(\cdot\mid s)$ be the mirror update in Algorithm~\ref{alg:svmd}.
Then for any $p\in\Delta(\mathcal{A})$,
\[
\eta\,\bigl\langle\widehat Q_h^t(s,\cdot),\,p-\pi_{t,h}(\cdot\mid s)\bigr\rangle
\;\leq\;
\mathrm{KL}\!\bigl(p\,\big\|\,\pi_{t,h}(\cdot\mid s)\bigr)
-\mathrm{KL}\!\bigl(p\,\big\|\,\pi_{t+1,h}(\cdot\mid s)\bigr)
+\frac{\eta^2 H^2}{2}.
\]
\end{lemma}
\begin{proof}
Fix $s$ and write $\pi_t,\pi_{t+1}$ for $\pi_{t,h}(\cdot\mid s),\pi_{t+1,h}(\cdot\mid s)$, $\widehat Q$ for $\widehat Q_h^t(s,\cdot)$.

\emph{Step 1: First-order optimality.}
Since $\pi_{t+1}$ maximizes $\eta\langle\widehat Q,p\rangle-\mathrm{KL}(p\|\pi_t)$ over $p\in\Delta(\mathcal{A})$ in the relative interior, the closed-form solution is
$\pi_{t+1}(a)\propto\pi_t(a)\exp(\eta\widehat Q(a))$, equivalently
\begin{equation}
\log\pi_{t+1}(a)-\log\pi_t(a)=\eta\widehat Q(a)-Z,
\label{eq:closed-form}
\end{equation}
where $Z$ is the (state-dependent) log-normalizer.

\emph{Step 2: Three-point identity.}
For any $p\in\Delta(\mathcal{A})$, direct computation gives the well-known
three-point identity for KL:
\[
\mathrm{KL}(p\|\pi_t)-\mathrm{KL}(p\|\pi_{t+1})-\mathrm{KL}(\pi_{t+1}\|\pi_t)
=\bigl\langle p-\pi_{t+1},\,\log(\pi_{t+1}/\pi_t)\bigr\rangle.
\]
Substituting~\eqref{eq:closed-form}, the right-hand side equals
$\eta\langle p-\pi_{t+1},\widehat Q\rangle$ (the $Z$ term vanishes since $p$ and $\pi_{t+1}$ are both probabilities). Therefore
\begin{equation}
\eta\bigl\langle\widehat Q,\,p-\pi_{t+1}\bigr\rangle
=\mathrm{KL}(p\|\pi_t)-\mathrm{KL}(p\|\pi_{t+1})-\mathrm{KL}(\pi_{t+1}\|\pi_t).
\label{eq:mirror-key}
\end{equation}

\emph{Step 3: Bound the drift $\langle\widehat Q,\pi_{t+1}-\pi_t\rangle$.}
By H\"older and Pinsker,
\[
\eta\bigl\langle\widehat Q,\pi_{t+1}-\pi_t\bigr\rangle
\leq \eta\|\widehat Q\|_\infty\|\pi_{t+1}-\pi_t\|_1
\leq \eta H\sqrt{2\,\mathrm{KL}(\pi_{t+1}\|\pi_t)},
\]
where $\|\widehat Q\|_\infty\leq H$ uses Assumption~\ref{asm:bdd}. By
$ab\leq \tfrac12 a^2+\tfrac12 b^2$ with $a=\eta H$, $b=\sqrt{2\,\mathrm{KL}(\pi_{t+1}\|\pi_t)}$,
\begin{equation}
\eta\bigl\langle\widehat Q,\pi_{t+1}-\pi_t\bigr\rangle
\leq \frac{\eta^2 H^2}{2}+\mathrm{KL}(\pi_{t+1}\|\pi_t).
\label{eq:drift}
\end{equation}

\emph{Step 4: Combine.}
Adding~\eqref{eq:mirror-key} and~\eqref{eq:drift},
\[
\eta\bigl\langle\widehat Q,\,p-\pi_t\bigr\rangle
\leq \mathrm{KL}(p\|\pi_t)-\mathrm{KL}(p\|\pi_{t+1})+\frac{\eta^2 H^2}{2},
\]
where the two $\mathrm{KL}(\pi_{t+1}\|\pi_t)$ terms cancel exactly. This is
the claim.
\end{proof}

\begin{lemma}[Weighted critic error]
\label{lem:critic}
Fix $t\in[T]$ and $h\in[H]$. Under
Assumptions~\ref{asm:bdd},~\ref{asm:real},~\ref{asm:smooth}, and conditional
on $\pi_t$, with probability at least $1-\delta'$ over the sampling of
$\{(s_{h,i},a_{h,i},G_{h,i})\}_{i=1}^{N_h}$,
\[
\mathbb{E}_{(s,a)\sim\nu_h^{\pi_t}}\!\bigl[(\widehat Q_h^t(s,a)-Q_h^{\pi_t}(s,a))^2\bigr]
\leq C\cdot\frac{H^2}{\varepsilon}\cdot\frac{\log(|\mathcal{F}_h|/\delta')}{N_h},
\]
for a universal constant $C>0$. Consequently, by Cauchy--Schwarz,
\[
\mathbb{E}_{(s,a)\sim\nu_h^{\pi_t}}\!\bigl[|\widehat Q_h^t(s,a)-Q_h^{\pi_t}(s,a)|\bigr]
\leq \sqrt{C\cdot\frac{H^2}{\varepsilon}\cdot\frac{\log(|\mathcal{F}_h|/\delta')}{N_h}}.
\]
\end{lemma}

\begin{proof}
\emph{Step 1: Reweighting to on-policy loss.}
For any $f\in\mathcal{F}_h$, the weighted population squared loss equals the
on-policy loss to a constant offset:
\[
\mathbb{E}_{s\sim\mu_h^t,a\sim\pi_{t,h},G}\!\bigl[w_h^t(s)(f(s,a)-G)^2\bigr]
=\mathbb{E}_{(s,a)\sim\nu_h^{\pi_t},G}\!\bigl[(f(s,a)-G)^2\bigr],
\]
because $w_h^t(s)=d_h^{\pi_t}(s)/\mu_h^t(s)$ exactly cancels the change of state measure, while $a\mid s\sim\pi_{t,h}$ and the conditional distribution of $G$ given $(s,a)$ are the same under both samplings.

\emph{Step 2: Realizability and noise structure.}
By construction, $G_{h,i}$ is a return-to-go obtained by rolling out $\pi_t$
from $(s_{h,i},a_{h,i})$, so
\[
\mathbb{E}[G_{h,i}\mid s_{h,i},a_{h,i}]=Q_h^{\pi_t}(s_{h,i},a_{h,i}),
\qquad
G_{h,i}\in[0,H].
\]
By Assumption~\ref{asm:real}, $Q_h^{\pi_t}\in\mathcal{F}_h$, so $Q_h^{\pi_t}$
minimizes the on-policy population squared loss over $\mathcal{F}_h$.

\emph{Step 3: Weighted ERM fast rate.}
The per-sample weighted squared loss is bounded:
$w_h^t(s)(f(s,a)-G)^2\in[0,H^2/\varepsilon]$ by
Assumption~\ref{asm:smooth} and Assumption~\ref{asm:bdd}. Standard
realizable-case ERM analysis (e.g.\ Bartlett--Mendelson local Rademacher, or
the offset-Rademacher bound of Liang--Rakhlin--Sridharan 2015) for finite
classes then gives
\[
\mathbb{E}_{(s,a)\sim\nu_h^{\pi_t}}\!\bigl[(\widehat Q_h^t(s,a)-Q_h^{\pi_t}(s,a))^2\bigr]
\leq C\cdot\frac{B\log(|\mathcal{F}_h|/\delta')}{N_h},
\]
with $B=H^2/\varepsilon$ the per-sample loss range. The Cauchy--Schwarz step
$\mathbb{E}|X|\leq\sqrt{\mathbb{E}X^2}$ then converts $L^2$ to $L^1$.
\end{proof}

\subsection{Proof of Theorem~\ref{thm:svmd}}

\paragraph{Step 1: Per-step bound, telescoping.}
Apply Lemma~\ref{lem:mirror} with $p=\mu_h(\cdot\mid s)$, take expectation
over $s\sim d_h^\mu$, and sum over $h\in[H]$ and $t\in[T]$. The KL terms
telescope in $t$:
\[
\eta\sum_{t=1}^T\sum_{h=1}^H \mathbb{E}_{s\sim d_h^\mu}\!\bigl[\langle\widehat Q_h^t(s,\cdot),\mu_h-\pi_{t,h}\rangle\bigr]
\leq D_\mu+\frac{T\eta^2 H^3}{2},
\]
since each per-step term contributes $\eta^2H^2/2$ and there are $TH$ of them. Dividing by $\eta T$,
\begin{equation}
\frac{1}{T}\sum_{t=1}^T\sum_{h=1}^H \mathbb{E}_{s\sim d_h^\mu}\!\bigl[\langle\widehat Q_h^t(s,\cdot),\mu_h-\pi_{t,h}\rangle\bigr]
\leq \frac{D_\mu}{\eta T}+\frac{\eta H^3}{2}.
\label{eq:opt-avg}
\end{equation}

\paragraph{Step 2: Performance difference and critic error.}
By Lemma~\ref{lem:pdl} and adding/subtracting $\widehat Q_h^t$,
\[
J(\mu)-J(\pi_t)
=\sum_{h=1}^H \mathbb{E}_{s\sim d_h^\mu}\!\bigl[\langle\widehat Q_h^t(s,\cdot),\mu_h-\pi_{t,h}\rangle\bigr]
+\sum_{h=1}^H\mathcal{E}_{h,t},
\]
where, with $e_{h,t}(s,a):=\widehat Q_h^t(s,a)-Q_h^{\pi_t}(s,a)$,
\[
\mathcal{E}_{h,t}:=\mathbb{E}_{s\sim d_h^\mu}\!\bigl[\langle Q_h^{\pi_t}(s,\cdot)-\widehat Q_h^t(s,\cdot),\mu_h-\pi_{t,h}\rangle\bigr].
\]

\paragraph{Step 3: Coverage transfers $\mathcal{E}_{h,t}$ to on-policy $L^1$.}
By the triangle inequality,
\[
|\mathcal{E}_{h,t}|
\leq \bigl|\mathbb{E}_{s\sim d_h^\mu,a\sim\mu_h}[e_{h,t}(s,a)]\bigr|
+\bigl|\mathbb{E}_{s\sim d_h^\mu,a\sim\pi_{t,h}}[e_{h,t}(s,a)]\bigr|.
\]
Assumption~\ref{asm:cov} gives $\nu_h^\mu\leq\kappa\nu_h^{\pi_t}$, hence
\[
\bigl|\mathbb{E}_{s\sim d_h^\mu,a\sim\mu_h}[e_{h,t}]\bigr|
\leq \kappa\,\mathbb{E}_{(s,a)\sim\nu_h^{\pi_t}}[|e_{h,t}|].
\]
Summing $\nu_h^\mu(s,a)\leq\kappa\nu_h^{\pi_t}(s,a)$ over $a$ also gives
$d_h^\mu\leq\kappa d_h^{\pi_t}$ pointwise, so
\[
\bigl|\mathbb{E}_{s\sim d_h^\mu,a\sim\pi_{t,h}}[e_{h,t}]\bigr|
\leq \kappa\,\mathbb{E}_{(s,a)\sim\nu_h^{\pi_t}}[|e_{h,t}|].
\]
Therefore
\begin{equation}
|\mathcal{E}_{h,t}|\leq 2\kappa\,\mathbb{E}_{(s,a)\sim\nu_h^{\pi_t}}[|e_{h,t}(s,a)|].
\label{eq:cov-transfer}
\end{equation}

\paragraph{Step 4: Apply Lemma~\ref{lem:critic} with a union bound.}
Apply Lemma~\ref{lem:critic} with confidence $\delta/(HT)$ for each $(t,h)$.
A union bound over the $HT$ such events implies that, with probability at
least $1-\delta$, simultaneously for all $t,h$,
\[
\mathbb{E}_{(s,a)\sim\nu_h^{\pi_t}}[|e_{h,t}(s,a)|]\leq \beta_h.
\]
Combined with~\eqref{eq:cov-transfer},
$|\mathcal{E}_{h,t}|\leq 2\kappa\beta_h$ on this event.

\paragraph{Step 5: Assemble.}
On the high-probability event of Step 4,
\[
J(\mu)-J(\pi_t)\leq\sum_{h=1}^H\mathbb{E}_{s\sim d_h^\mu}\!\bigl[\langle\widehat Q_h^t,\mu_h-\pi_{t,h}\rangle\bigr]+2\kappa\sum_{h=1}^H\beta_h.
\]
Averaging over $t\in[T]$ and using
$J(\bar\pi_T)=\frac1T\sum_{t=1}^T J(\pi_t)$ (since $\bar\pi_T$ is the uniform
mixture), then plugging in~\eqref{eq:opt-avg},
\[
J(\mu)-J(\bar\pi_T)\leq\frac{D_\mu}{\eta T}+\frac{\eta H^3}{2}+2\kappa\sum_{h=1}^H\beta_h.
\]
The choice $\eta=\sqrt{2D_\mu/(H^3 T)}$ minimizes the first two terms,
giving the second form of the bound. \hfill$\square$
\section{Experiments}
\label{app:experiments}

This appendix provides additional details for the experiments in Section~4. 
The purpose of these experiments is to test whether variance-driven state resampling is most useful in the regime predicted by our analysis: tasks in which rollouts from the initial state often produce uniform rewards, but intermediate states along successful or partially successful trajectories can still produce mixed continuation outcomes. 
Across all experiments, VSR changes only the state from which rollout groups are initialized; the underlying policy-gradient objective, reward function, and evaluation protocol are otherwise kept the same as the corresponding baseline.

\subsection{Toy gridworld}
\label{app:toy-gridworld}

We use the gridworld experiment in Section~2.3 as a controlled setting where both the on-policy state-visitation distribution $d^{\pi_t}$ and the within-state reward standard deviation $\sigma_{\pi_t}$ can be computed exactly. 
The environment is a $7 \times 7$ gridworld with fixed obstacles, a start state $S$, and a goal state $G$. 
The agent receives a binary terminal reward $R \in \{0,1\}$ indicating whether it reaches the goal. 
The policy is a one-hidden-layer MLP with 32 hidden units and tanh activations, followed by a four-action output layer.

At each training iteration, we enumerate the finite state space to compute the current occupancy distribution $d^{\pi_t}(s)$ and the continuation success probability $p_t(s)$ from every reachable state. 
We then compute
\[
    \sigma_{\pi_t}(s) = \sqrt{p_t(s)(1-p_t(s))}
\]
and form the variance-driven proposal
\[
    \mu_t(s) \propto d^{\pi_t}(s)\sigma_{\pi_t}(s).
\]
We compare policy-gradient updates that sample states from the standard on-policy proposal $d^{\pi_t}$ against updates that sample from $\mu_t$. 
This experiment is intended as a diagnostic visualization of the variance-optimal proposal rather than as a benchmark of practical replay heuristics: unlike the LLM and agent experiments, no approximation to either $d^{\pi_t}$ or $\sigma_{\pi_t}$ is used.


\subsection{Mathematical reasoning}
\label{app:math-experiments}

\paragraph{Training data.}
For mathematical reasoning, we train on a hard subset of OpenR1-Math-220k. 
The subset contains 1K prompts selected to have low pass rate under the base policy. 
We filter for prompts with pass@32 at most 1 using Qwen2.5-7B-Instruct. 
This filtering focuses training on problems for which prompt-level on-policy rollouts are usually uninformative, creating a setting where intermediate successful prefixes can provide useful continuation states.

\paragraph{Models and baselines.}
We run experiments with Qwen2.5-7B-Instruct and Llama-3.2-3B-Instruct. 
For each model, we compare the base model, supervised fine-tuning on expert trajectories, standard GRPO, and VSR+GRPO. 
The SFT baseline uses the same successful expert trajectories that are used to initialize the VSR replay buffer, so that the comparison separates direct imitation of expert completions from using expert prefixes only as reset states for RL rollouts. All experiments are performed on 4 A100 GPUs.

\paragraph{VSR initialization.}
In the math experiments, VSR uses cold-start initialization because successful trajectories are rare under the base policy. 
We seed the replay buffer with prefix states extracted from successful DeepSeek-R1 trajectories. 
A seeded prefix consists of the original prompt concatenated with a partial reasoning trace. 
During VSR training, the prefix tokens are treated as fixed context and are masked out of the policy-gradient loss; only newly generated continuation tokens contribute to the GRPO objective. 
As training proceeds, the replay buffer is updated with successful prefixes generated by the current policy, and recency-based eviction gradually replaces expert-seeded prefixes with on-policy states. For every prompt, we only keep the most recent single successful trajectory.

\paragraph{Replay construction and scoring.}
For every successful fresh trajectory $\tau$ generated during training, we insert prefix states $s_i=(x,y_{<i})$ into the replay buffer. 
Each inserted prefix receives the zero-cost variance score
\[
    \hat{\sigma}(s_i)
    =
    \sqrt{\hat{p}_i(1-\hat{p}_i)},
    \qquad
    \hat{p}_i=(1-\alpha_i)p_0(x)+\alpha_i R(\tau),
    \qquad
    \alpha_i=i/|\tau|.
\]
Here $p_0(x)$ is the empirical success rate of the rollout group generated from the original prompt, and $R(\tau)$ is the terminal reward of the trajectory from which the prefix was extracted. 
At each update, a fraction $\rho$ of batch elements is drawn from the replay buffer with probability proportional to $\hat{\sigma}$, while the remaining batch elements are fresh prompts from the training set. 
For replayed prefixes, we generate a fresh GRPO rollout group from the prefix using the current policy.

\begin{figure}[h]
    \centering
    \includegraphics[width=0.4\linewidth]{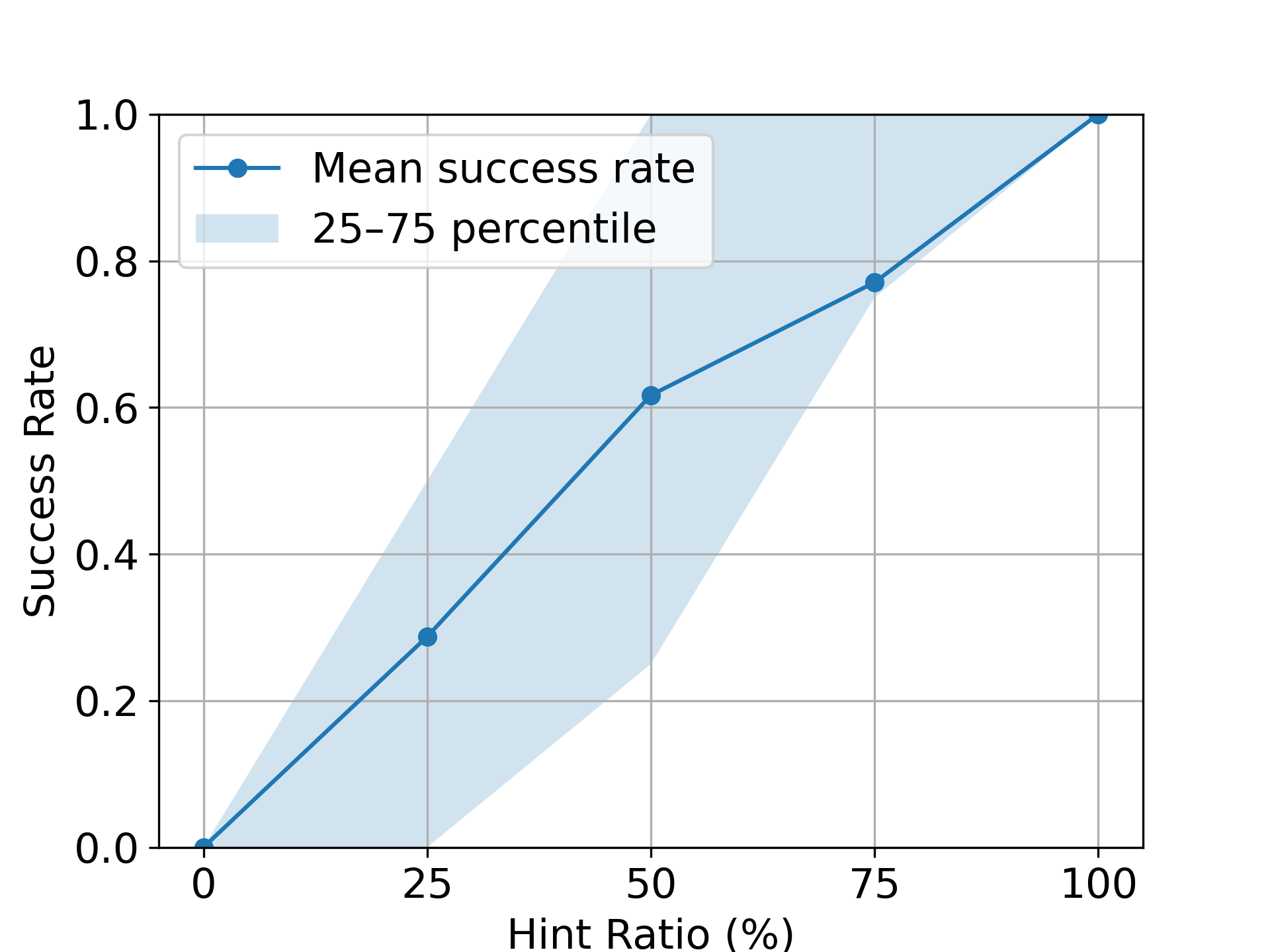}
    \caption{\textbf{Empirical continuation success rate as a function of hint ratio.} For each of the 1K hard prompts in the training subset, we truncate the corresponding successful DeepSeek-R1 trajectory to a fraction $\alpha \in \{0\%, 25\%, 50\%, 75\%, 100\%\}$ of its length and sample 8 continuations under Qwen2.5-7B-Instruct. The blue curve shows the mean per-prompt success rate across the 1K prompts; the shaded band shows the 25--75 percentile across prompts. The mean is approximately linear in $\alpha$, supporting the linear interpolation used in $\hat{p}_t(s_i \mid \tau) = (1 - \alpha_i)\,p_0(x) + \alpha_i\, R(\tau)$.}
    \label{fig:hint_ratio}
\end{figure}

\paragraph{Empirical justification of the linear interpolation.}
The score $\hat{\sigma}_t$ in Algorithm~\ref{alg:vsr} estimates the continuation success probability at an intermediate prefix $s_i$ by linearly interpolating between the prompt-level rollout success rate $p_0(x)$ and the terminal reward $R(\tau)$ of the trajectory the prefix is drawn from. We adopt the linear form because it is the simplest functional form consistent with the boundary conditions $\hat{p}_0 = p_0(x)$ and $\hat{p}_{|\tau|} = R(\tau)$, and we did not tune it. To check that this choice is not grossly mis-specified in our training setting, we measure the empirical continuation success rate as a function of the hint ratio $\alpha = i / |\tau|$. Concretely, for each of the 1K hard prompts in our training subset---on which the base policy's pass rate is essentially zero---we take the corresponding successful DeepSeek-R1 trajectory used to seed the replay buffer, truncate it to $\alpha \in \{0\%, 25\%, 50\%, 75\%, 100\%\}$ of its length, and sample 8 continuations under Qwen2.5-7B-Instruct from each truncation. Figure~\ref{fig:hint_ratio} reports the mean and 25--75 percentile of the resulting per-prompt success rate across the 1K prompts. The mean success rate rises approximately linearly in $\alpha$ from $\approx 0$ at $\alpha = 0$ to $\approx 1$ at $\alpha = 100\%$, with no strong concavity or saturation that would favor a sigmoid or quadratic interpolant. The interquartile spread is widest near the middle of the trajectory, which is exactly the regime where $\sigma_\pi(s)$ is largest under the proxy and where VSR is designed to concentrate sampling mass. We therefore retain the linear form throughout; tighter low-cost surrogates remain a natural target for future work.

\paragraph{Evaluation.}
We evaluate final models on AIME24, AIME25, AMC23, OlympiadBench, MATH500, and Minerva. 
For each benchmark, we report accuracy under the same decoding and evaluation protocol across all methods. We sample 16 rollouts over 2 different seeds for all the prompts under \texttt{temperature=0.7} and report the mean results.
In addition to final benchmark accuracy, we track training reward, validation accuracy, and response length over training. 
Response length is included because VSR can reduce the number of newly generated tokens per update by starting a fraction of rollouts from mid-trajectory prefixes.

\begin{table}[t]
\centering
\caption{Hyperparameters for mathematical reasoning experiments.}
\label{tab:math-hparams}
\begin{tabular}{lll}
\toprule
Hyperparameter & Qwen2.5-7B-Instruct & Llama-3.2-3B-Instruct \\
\midrule
Replay ratio $\rho$ & 0.5 & 0.5 \\
Training steps & 360 & 360 \\
Max Prompt Length & 2048 & 2048 \\
Max Response Length & 4096 & 4096 \\
Rollout Group Size & 8 & 8 \\
Training batch size & 128 & 128 \\
Mini batch size & 64 & 64 \\
Learning rate & $1\times 10^{-6}$ & $1\times 10^{-6}$ \\
KL coefficient & 0 & 0 \\
\bottomrule
\end{tabular}
\end{table}

\subsection{Long-horizon agentic tasks}
\label{app:agentic-experiments}

\paragraph{Environments.}
We evaluate VSR on both Sokoban and BabyAI. 
These environments require multi-step decision making and produce sparse terminal rewards, making them a natural testbed for state-level resampling. 
For BabyAI, we use the SynthLoc, OneRoom, and MiniBoss levels. 
The policy observes a serialized interaction history consisting of the task instruction, previous actions, environment feedback, and current observations. 
A replay state is therefore represented as a prefix of this serialized trajectory, analogous to a token prefix in the math experiments.

\paragraph{Models and baselines.}
We use Qwen2.5-1.5B-Instruct and Qwen2.5-3B-Instruct. 
The base RL algorithm is GiGPO, and VSR is added as a plug-in state-resampling component. 
We compare the base model, GiGPO, and VSR+GiGPO under matched model, environment, reward, and training settings. 
Unlike the math experiments, we do not use expert successful trajectories for cold-start initialization. 
The replay buffer is populated entirely from trajectories generated by the current policy. All experiments are performed on a single H200 GPU.

\paragraph{VSR in agentic environments.}
At each training iteration, the policy generates rollout groups either from the initial environment state or from replayed intermediate states. 
When a successful trajectory is observed, we insert its intermediate serialized states into the replay buffer. 
Each replayed state includes all information needed to resume interaction from that point, including prior actions and environment observations. 
As in the math setting, prefix tokens or history tokens preceding the replay point are treated as fixed context and are masked from the policy-gradient loss. 
The model is optimized only on newly generated actions after the replay state.

Prefix scores are computed with the same zero-cost interpolation rule used in the math experiments. 
For a prefix at normalized position $\alpha_i$ along a successful trajectory, the estimated continuation success probability interpolates between the rollout-group success rate from the initial state and the realized terminal reward of the trajectory. 
Replay prefixes are sampled in proportion to the resulting $\hat{\sigma}$ score. 
The recency window is used to keep the replay buffer aligned with the current policy, preventing the proposal from becoming dominated by states collected under earlier, weaker policies.

\paragraph{Evaluation.}
For Sokoban and BabyAI, we report task success rate. For Sokoban, we sample 128 times over a fixed set of map seeds and report the average success rate. For each level of Baby AI, we sample 128 times over a fixed set of map seeds and report the average success rate. The environment hyperparameters remain the same as what we use in training.
For training dynamics, we additionally track training reward, validation success rate, and episode length. 
Episode length is reported because state replay changes the average number of newly generated steps per rollout: when training starts from intermediate states, the model can receive useful gradient signal while generating fewer new actions.

\begin{table}[t]
\centering
\caption{Hyperparameters for agentic-task experiments.}
\label{tab:agentic-hparams}
\begin{tabular}{lll}
\toprule
Hyperparameter & Qwen2.5-1.5B-Instruct & Qwen2.5-3B-Instruct \\
\midrule
Replay ratio & 0.3 & 0.3 \\
Recency window & 5 & 5 \\
Max prompt length & 2048 & 2048 \\
Max response length & 256 & 256 \\
Rollout group size & 8 & 8 \\
Training batch size & 32 & 32 \\
Learning rate & $1\times 10^{-6}$ & $1\times 10^{-6}$ \\
KL coefficient & 0 & 0 \\
\bottomrule
\end{tabular}
\end{table}

\subsection{Ablation details}
\label{app:ablation-details}

We conduct ablations on BabyAI using the same training and evaluation setup as the main agentic-task experiments. 
All ablation variants are built on top of GiGPO and differ only in the replay component.

\paragraph{Uniform replay.}
The variant VSR w/o scoring keeps the same replay buffer construction and recency-based updates as VSR, but samples prefixes uniformly from the buffer rather than in proportion to $\hat{\sigma}$. 
This ablation tests whether the variance-driven score contributes beyond the benefit of simply restarting rollouts from previously visited successful prefixes.

\paragraph{Static replay buffer.}
The variant VSR w/o recency keeps $\hat{\sigma}$-weighted prefix sampling but disables recency-based buffer updates. 
The buffer is populated early in training and then held fixed. 
This ablation tests whether tracking the current policy's state-visitation distribution is necessary for sustained learning. 
A static buffer can be useful early in training, but as the policy improves, its state distribution changes; without recency updates, replayed prefixes may no longer approximate the current $d^{\pi_t}$ factor in the variance-driven proposal.

\paragraph{Full VSR.}
Full VSR combines both components: recent successful-prefix replay approximates the current state-visitation factor, while $\hat{\sigma}$-weighted sampling prioritizes prefixes whose continuations are estimated to have high residual uncertainty. 
The ablations therefore isolate the two practical approximations to the proposal $\mu_{\mathrm{var}}^\pi \propto d^\pi \sigma_\pi$.


\subsection{Discussion-period controls and validations}
\label{app:rebuttal-controls}

\begin{table}[h]
\centering
\caption{Additional mathematical-reasoning controls on Llama-3.2-3B-Instruct. These are single training runs; the original benchmark evaluation protocol is otherwise unchanged.}
\label{tab:math-controls-appendix}
\scriptsize
\begin{tabular}{lccccccc}
\toprule
\textbf{Method} & \textbf{AIME24} & \textbf{AIME25} & \textbf{AMC23} & \textbf{Olympiad} & \textbf{MATH500} & \textbf{Minerva} & \textbf{Avg.} \\
\midrule
GRPO & 4.4 & 1.0 & 18.7 & 12.6 & 39.2 & 16.5 & 15.4 \\
Reinforce-Ada-Est & 4.2 & 0.1 & 18.0 & 12.4 & 39.9 & 15.4 & 15.0 \\
Static uniform expert-prefix resets & 5.6 & 0.5 & 22.7 & 14.1 & 43.7 & 15.8 & 17.1 \\
VSR & 6.2 & 0.6 & 20.4 & 15.3 & 54.0 & 17.8 & 19.1 \\
\bottomrule
\end{tabular}
\end{table}

\paragraph{Static expert-prefix reset control.}
For Llama-3.2-3B-Instruct, we added a control using the same successful expert-trajectory source as VSR but sampling uniformly from a fixed expert-prefix pool and never adding prefixes generated by the current policy. The benchmark results (AIME24/AIME25/AMC23/OlympiadBench/MATH500/Minerva/Avg.) are 4.4/1.0/18.7/12.6/39.2/16.5/15.4 for GRPO, 5.6/0.5/22.7/14.1/43.7/15.8/17.1 for static uniform expert-prefix resets, and 6.2/0.6/20.4/15.3/54.0/17.8/19.1 for VSR. Thus privileged reset states themselves explain a material part of the gain. Because full VSR also adds online policy-generated replay, recency updates, and priority-weighted sampling, the remaining difference cannot be attributed specifically to the priority score. This is one training run per method.

\paragraph{Prompt-level adaptive allocation.}
We also compare with Reinforce-Ada-Est on Llama-3.2-3B-Instruct. Each update uses 128 prompts and exactly $128\times8=1024$ generated responses. Reinforce-Ada-Est reallocates these responses across prompts subject to $N_x\in[4,32]$ and $\sum_x N_x=1024$, so its average group size is 8, matching GRPO and VSR. It uses no expert trajectories, expert prefixes, or intermediate reset states; every rollout begins from the original problem prompt. The training subset, number of updates, optimization hyperparameters, and evaluation protocol are otherwise matched. Under this matched response budget, its benchmark results are 4.2/0.1/18.0/12.4/39.9/15.4 (Avg. 15.0), compared with GRPO at 15.4 and VSR at 19.1. This supports the narrower conclusion that prompt-level response reallocation alone does not reproduce the full VSR result in this hard-prompt setting; it does not isolate state-level allocation as the sole source of the difference, and it is a single run.

\paragraph{Independent-run scoring ablation.}
On Qwen2.5-1.5B BabyAI, the three uniform-replay runs are 46.1, 39.1, and 46.9 (mean $44.0\pm4.3$ SD), while the three VSR runs are 47.6, 50.8, and 46.1 (mean $48.2\pm2.4$ SD). We report all runs and do not perform formal significance testing.

\begin{table}[h]
\centering
\caption{Independent-run BabyAI scoring ablation on Qwen2.5-1.5B. Values are average success rates (\%).}
\label{tab:seed-ablation-appendix}
\small
\begin{tabular}{lcccc}
\toprule
\textbf{Method} & \textbf{Run 1} & \textbf{Run 2} & \textbf{Run 3} & \textbf{Mean $\pm$ SD} \\
\midrule
Uniform replay & 46.1 & 39.1 & 46.9 & $44.0\pm4.3$ \\
VSR & 47.6 & 50.8 & 46.1 & $48.2\pm2.4$ \\
\bottomrule
\end{tabular}
\end{table}

\paragraph{Failed-prefix analysis and same-checkpoint continuations.}
For failed trajectories, the mean continuation success rate decreases as the prefix becomes longer (0.396, 0.260, and 0.130 at the three tested prefix positions), while the fraction of continuation groups with no successful rollout increases from 26\% to 42\% to 62\%. For successful trajectories from the same problems, the corresponding all-fail rates are 0\%, 2\%, and 4\%. At the latest tested failed-prefix position, 36\% of groups still contain mixed outcomes, so this evidence does not imply that every failed prefix is unrecoverable; it supports successful-only filtering as a budget-allocation choice in the sparse-reward regime.

\paragraph{Statewise validation of the priority score.}
We directly validate the implemented score on MATH-500. For each policy, we select 50 problems with mixed initial outcomes, retain one self-generated successful trajectory per problem, extract prefixes at 25\%, 50\%, and 75\% of the response, and generate 16 new continuations from each prefix, yielding 150 prefix states per policy. For Qwen2.5-7B-Instruct, predicted versus empirical continuation standard deviation has pooled Pearson correlation 0.441 [95\% CI: 0.277, 0.591] and Spearman correlation 0.433 [0.270, 0.585]; Pearson correlations at the three fixed prefix positions are 0.446/0.387/0.342 and Spearman correlations are 0.494/0.305/0.329. For an intermediate training checkpoint, pooled Pearson is 0.484 [0.291, 0.641] and pooled Spearman is 0.499 [0.314, 0.642], with within-position Pearson 0.587/0.432/0.307 and Spearman 0.620/0.424/0.315. These moderate positive associations, including within fixed prefix positions, show that the score contains state-dependent information beyond prefix position alone. The validation uses selected successful MATH-500 states rather than the full online replay-buffer distribution, and the score remains an uncalibrated heuristic rather than an unbiased or consistent variance estimator.

\paragraph{Generated-token accounting.}
Across all model calls in the agentic experiments, Qwen2.5-1.5B VSR+GiGPO generates 39.27M tokens over 150 matched updates compared with 81.30M for GiGPO; Qwen2.5-3B generates 34.43M versus 36.26M over 100 matched updates. These totals establish lower generated-token usage at matched update counts. We do not claim lower wall-clock cost or performance superiority under an exactly matched total-token budget.




\end{document}